%% file: main_arxiv.tex
\documentclass[11pt]{article}
\usepackage{fullpage}
\usepackage{amsthm}
\usepackage{amsmath}
\usepackage{amssymb}
\usepackage[numbers]{natbib}
\usepackage[algo2e]{algorithm2e}

\usepackage{color}
\usepackage[english]{babel}
\usepackage{graphicx}
\usepackage{grffile}
\usepackage{wrapfig}
\usepackage{url}
\usepackage{color}
\usepackage{epstopdf}
\usepackage[T1]{fontenc}
\usepackage{bbm}
\usepackage{comment}
\usepackage{dsfont}
\usepackage{thm-restate}
\usepackage{bm}
\usepackage{tcolorbox}
\usepackage{algorithm}
\usepackage{algpseudocode}

\graphicspath{{./figs/}}
\usepackage{mathtools}

\newtheorem{theorem}{Theorem}[section]
 
\newtheorem{lemma}[theorem]{Lemma}
\newtheorem{definition}[theorem]{Definition}

\newtheorem{assumption}[theorem]{Assumption}

\newtheorem{remark}[theorem]{Remark}

\newcommand{\wh}{\widehat}
\newcommand{\wt}{\widetilde}

\newcommand{\eps}{\epsilon}

\newcommand{\R}{\mathbb{R}}

\renewcommand{\varepsilon}{\epsilon}
\renewcommand{\tilde}{\wt}
\renewcommand{\hat}{\wh}

\renewcommand{\eps}{\epsilon}

\newcommand{\W}{\mathsf{W}}
\newcommand{\V}{\mathsf{V}}

\newcommand{\rand}{\mathsf{rand}}
\newcommand{\column}{\mathsf{column}}
\newcommand{\row}{\mathsf{row}}
\newcommand{\DPGD}{\mathsf{DPGrad}}

\newcommand{\D}{\mathcal{D}}

\DeclareMathOperator*{\E}{{\mathbb{E}}}

\DeclareMathOperator{\init}{init}

\newcommand{\Andrej}[1]{{\color{red}[Andrej: #1]}}

\makeatletter
\newcommand*{\RN}[1]{\expandafter\@slowromancap\romannumeral #1@}
\makeatother

\title{Continual learning: a feature extraction formalization, \\
an efficient algorithm, and fundamental obstructions}
\author{
    Binghui Peng \\
    Columbia University \\
    \texttt{bp2601@columbia.edu} \\
    \and
    Andrej Risteski \\
    Carnegie Mellon University \\
    \texttt{aristesk@andrew.cmu.edu} \\}
    
\begin{document} 
\maketitle

\begin{abstract}%

Continual learning is an emerging paradigm in machine learning, wherein a model is exposed in an online fashion to data from multiple different distributions (i.e. environments), and is expected to adapt to the distribution change. Precisely, the goal is to perform well in the new environment, while simultaneously retaining the performance on the previous environments (i.e. avoid ``catastrophic forgetting'')---without increasing the size of the model.

While this setup has enjoyed a lot of attention in the applied community, there hasn't be theoretical work that even formalizes the desired guarantees. In this paper, we propose a framework for continual learning through the framework of feature extraction---namely, one in which features, as well as a classifier, are being trained with each environment. When the features are linear, we design an efficient gradient-based algorithm $\DPGD$, that is guaranteed to perform well on the current environment, as well as avoid catastrophic forgetting. In the general case, when the features are non-linear, we show such an algorithm cannot exist, whether efficient or not.
\end{abstract}

\input{introduction}

\input{preliminary}
\input{relate}

\input{algorithm}

\input{analysis}
\input{lower}
\input{discussion}

\clearpage
\newpage
\bibliographystyle{plain}
\bibliography{ref}

\newpage
\appendix
\input{appendix-upper}
\input{appendix-lower}

\end{document}

%% file: introduction.tex
\section{Introduction}
\label{sec:intro}

In the last few years, there has been an increasingly large focus in the modern machine learning community on settings which go \emph{beyond iid data}. This has resulted in the proliferation of new concepts and settings such as out-of-distribution generalization \citep{hendrycks2021many},  domain generalization \citep{blanchard2011generalizing}, multi-task learning \citep{zhang2018overview}, continual learning \citep{parisi2019continual} and etc. 
{\em Continual learning}, which is the focus of this paper, concerns learning through a sequence of environments, with the hope of retaining old knowledge while adapting to new environments.

Unfortunately, despite a lot of interest in the applied community---as evidenced by a multitude of NeurIPS and ICML workshops \citep{workshop2018, workshop2020, workshop2021}---approaches with formal theoretical guarantees are few and far between. 
The main reason, similar encountered as its cousin fields like out-of-distribution generalization or multi-tasks learning, usually come with some ``intuitive'' desiderata --- but no formal definitions. 
What's worse, it's often times clear that without strong data assumptions---the problem is woefully ill-defined. 

The intuitive desiderata the continual learning community has settled on is that the setting involves cases where an algorithm is exposed (in an online fashion) to data sequentially coming from different distributions (typically called ``environments'', inspired from a robot/agent interacting with different environments). Moreover, the goal is to keep the size of the model being trained fixed, and make sure the model performs well on the current environment \emph{while simultaneously} maintaining a good performance in the previously seen environments. In continual learning parlance, this is termed ``resistance to catastrophic forgetting''. 

It is clear that some of the above desiderata are shared with well-studied learning theory settings (e.g. online learning, lifelong learning), while some aspects differ. For example, in online learning, we don't care about catastrophic forgetting (or we only do so in some averaged sense); in lifelong learning, it's not necessary to keep the size of the model fixed.
It is also clear that absent some assumptions on the data and the model being trained, these desiderata cannot possibly be satisfied: why would there even exist a model of some fixed size that performs well on both past environments, and current ones --- let alone one that gets updated in an online fashion.  

\paragraph{A feature-extraction formalization of continual learning:} Our paper formalizes a setting for continual learning through the lens of \emph{feature extraction}: the model maintains a \emph{fixed} number of (trainable) features, as well as a \emph{linear classifier} on top of said features. The features are updated for every new environment, with the objective that the features are such that a good linear classifier exists for the \emph{new} environment, while the previously trained linear classifiers (on the updated features) are still good for the past environments. The reason the linear classifiers from previous rounds are not allowed to be updated is storage efficiency: in many settings, training data from previous environments is discarded, thus the linear classifier cannot be fine-tuned in subsequent rounds to the updated features. The number of features is kept fixed for the same reason: if we are to learn new features for every new environment, the model size (and hence storage requirements) would grow.

We prove two main results for our setting. 
\begin{enumerate} 
\item When the features are a linear function of the input data, and a good set of features exist, we design \emph{an efficient algorithm}, named doubly projected gradient descent, or $\DPGD$, that has a good accuracy on all environments, and resists catastrophic forgetting. Our algorithm, while being novel, bears some resemblance to a class of projection-based algorithms used in practice \citep{farajtabar2020orthogonal, chaudhry2020continual} --- and we hope some of the insights might port beyond the linear setting.

\item When the features are allowed to be a non-linear function of the input, perhaps surprisingly, continual learning is not possible in general. That is: even if a good set of features exists, the online nature of the setting, as well as the fact that the linear classifiers for past environments are not allowed to be updated, makes it possible for the algorithm to ``commit'' to linear classifiers, such that either catastrophic forgetting, or poor performance on the current environment has to occur. 
\end{enumerate}

%% file: preliminary.tex
\section{Our results}
\label{sec:result}

\subsection{Setup}
In a continual learning problem, the learner has sequential access to $k$ environments. In the $i$-th ($i \in [k]$) environment, the data is drawn i.i.d. from the underlying distribution $\D_i$ over $\R^{d} \times \R$, denoted as $(x, y) \sim \D_i$, where $x\in \R^{d}$ is the input and $y \in \R$ is the label. 
Motivated by the empirical success of representation learning \citep{bengio2013representation,devlin2019bert},
we formulate the continual learning problem through the feature extraction view: 
The learner is required to learn a common feature mapping (also known as representation function) $R: \R^{d} \rightarrow \R^{r}$ that maps the input data $x\in \R^{d}$ to a low dimensional representation $R(x) \in \R^{r}$ ($r \ll d$), together with a sequence of task-dependent linear classifiers (also known as linear prompts) $v_1, \ldots, v_k \in \R^{r}$ that place on top of the representation. That is to say, the prediction of the $i$-th environment is made by $f(x) = \langle v_i, R(x)\rangle$. 

As this is the first-cut study, we focus on the {\em realizable} and the {\em proper} learning setting.\footnote{It is possible to extend our algorithmic result to the non-realizable setting, when the label has symmetric sub-gaussian noise.} That is, we assume the existence of a feature mapping $R$ in the function class $\mathcal{H}$ (which is known in advance) and a sequence of linear predictor $v_1, \ldots, v_k$ such that for any $i \in [k]$ and any data $(x, y) \sim \D_i$, $y = \langle v_i, R(x)\rangle$ (realizable).
The learner is required to output a function $R$ that belongs to the hypothesis class $\mathcal{H}$ (proper).

\begin{remark}[Known environment identity]
Our model requires the knowledge of environment identity at test time, and thus can be classified into the category of incremental task learning.
We note there are also empirical research focusing on unknown environment identity, for which we left as future work (See Section \ref{sec:discussion}).
\end{remark}

The guarantee that we wish our learning algorithm to obtain is as follows:

\begin{definition}[Continual learning]
\label{def:prob-non-linear}
Let $d, k, r \in \mathbb{N}$, $r \ll d, k$, $\eps \in (0, 1/2)$. Let $\mathcal{H}$ be a function class consists of feature mapping from $\R^d$ to $\R^{r}$. The continual learning problem is defined over $k$ environments $\D_1, \ldots, \D_k$. There exists a function $R^{\star} \in \mathcal{H}$ and a sequence of linear classifiers $v_1^{\star}, \ldots, v_k^{\star} \in \R^{r}$ such that for any $(x, y) \sim \D_{i}$ ($i \in [k]$), the label satisfies $y = \langle v_i^{\star}, R^{\star}(x)\rangle$.

The continual learner has sequential access to environments $\D_1, \ldots, \D_k$ and can draw arbitrary many samples. The goal is to learn a representation function $R \in \mathcal{H}$ and a sequence of linear prompts $v_1, \ldots, v_k \in \R^{r}$ that achieve valid accuracy on the current task and do not suffer from catastrophic forgetting. Formally, in the $i$-th environment ($i \in [k]$), the learner optimizes the feature mapping $R$ and the linear classifier $v_i$ (without changing $v_1, \ldots, v_{i-1}$) and aims to satisfy
\begin{itemize}
\item {\bf Avoid catastrophic forgetting}: During the execution of the $i$-th task, the algorithm guarantees that 
\begin{align*}
L(R, v_j) := \frac{1}{2}\E_{(x, y)\sim \D_j}(\langle v_j, R(x)\rangle - y)^2 \leq \eps ~~ \text{ for all } j = 1, \ldots, i -1,
\end{align*}
\item {\bf Good accuracy on the current task:} At the end of $i$-th task, the algorithm guarantees that 
\begin{align*}
L(R, v_i) := \frac{1}{2}\E_{(x, y)\sim \D_i}(\langle v_i, R(x)\rangle  - y)^2 \leq \eps.
\end{align*}
\end{itemize}

For linear feature mapping, the representation function can be written in a linear form $R(x) = U^{\top}x$ for some $U \in \R^{d\times r}$, and it implies the $i$-th environment is generated by a linear model. That is, defining $w_i = Uv_i \in \R^{d}$, one can write $y = \langle v_i, U^{\top} x\rangle = \langle w_i, x\rangle$. 
\end{definition}

\begin{remark}[The benefit of continual learning with linear feature] Note, for linear features, while it is always possible to learn the sequence of linear classifier $w_1, \ldots, w_k \in \R^{d}$ separately,  continual learning offers (1) memory efficiency ($O(kr + dr)$ vs. $O(dk)$); (2) sample efficiency ($O(r)$ vs. $O(d)$ samples per task in the asymptotic regime $k \rightarrow \infty$).
\end{remark}

\subsection{DPGrad: Efficient gradient based method for linear features}

For the case of linear features, we propose an efficient algorithm which we term $\DPGD$ (pseudocode in Algorithm \ref{algo:ogd}), which is an efficient gradient based method and provably learns the representation while avoids catastrophic forgetting. Towards stating the result, we make a few technical assumptions.

\begin{assumption}[Distribution assumption]
\label{asp:distribution}
For any $i \in [k]$, we assume $\D_i$ has zero means and it is in isotropic position, that is, $\E_{x\sim \D_i}[x] = \vec{0}$ and $\E_{x\sim \D_i}[xx^{\top}] = I$.  
\end{assumption}

\begin{remark}
This assumption is largely for convenience. In fact, one can replace the isotropic condition with a general bounded covariance assumption, our algorithm still can work with extra preprocessing step, and the sample complexity blows up with the condition number of covariance matrix.
\end{remark}

\begin{assumption}[Range assumption]
\label{asp:range}
For any $i \in [k]$, $w_i$ has bounded norm, i.e., $\|w_i\|_2 \leq D$.
\end{assumption}

\begin{assumption}[Signal assumption]
\label{asp:signal}
For any $i \in [k]$, let $\W_i = \mathsf{span}(w_1, \ldots, w_{i})$, $\W_{i,\perp}$ be the space perpendicular to $W$ and $P_{\W_{i}}$, $P_{ \W_{i,\perp}}$ be the projection operator. We assume either $w_i$ belongs to $\W_{i-1}$ or it has non-negligible component orthogonal to $\W_{i-1}$, i.e., $\|P_{\W_{i-1,\perp}}w_i\|_2 \in \{0\} \cup [1/D, D]$.
\end{assumption}

\begin{assumption}[Bit complexity assumption]
\label{asp:bit}
Each coordinate of $w_i$ is a multiple of $\nu > 0$. 
\end{assumption}

\begin{remark}
The Range assumption ensures an upper bound on $\|w_i\|_2$, which is standard in statistical learning setups.  The Signal assumption ensures that $w_{i}$ provides enough signal if it has not been ``covered'' by $\W_{i-1}$.
\end{remark}

\begin{remark}
The bit complexity assumption states that $w_i$ can be described with finite bits, and is mostly for convenience --- namely so we can argue we exactly recover $w_i$---which makes calculations involving projections of features learned in the past cleaner. 
Since the number of gradient iterations only depends logarithmically on $\nu$, one can remove the assumption by having a quasi-polynomially small estimation error for $w_i$. 
\end{remark}

The main result is then as follows: 
\begin{restatable}[Continual learning with linear feature]{theorem}{linearThm}
\label{thm:linear}
Let $k, d, r \in \mathbb{N}$, $r\ll k, d$, $\eps \in (0, 1/2)$. When the features are a linear function over the input data, under Assumption \ref{asp:distribution} and Assumption \ref{asp:range}-\ref{asp:bit}, with high probability, $\DPGD$ provably learns the feature mapping and avoids catastrophic forgetting. In particular, during the execution of $i$-th environment,  $\DPGD$ always guarantee
\begin{align}
\label{eq:forget1}
L(U, v_j) := \frac{1}{2}\E_{(x, y)\sim \D_j}(x^{\top}Uv_j - y)^2 \leq \eps, ~~ \text{ for all } j = 1,2, \ldots, i-1,
\end{align}
and at the end of $i$-th environment,  $\DPGD$ ensures
\begin{align}
\label{eq:acc1}
L(U, v_i) = \frac{1}{2}\E_{(x, y)\sim \D_j}(x^{\top}Uv_i - y)^2 \leq \eps.
\end{align}
\end{restatable}

\subsection{Fundamental obstructions for non-linear features}

Learning with a non-linear feature mapping turns out to be much more difficult --- even without computational constraints---and our result rules out the existence of a proper continual learner, even when all environment distributions are uniform and the representation function is realizable by a two-layer convolutional neural network. 
\begin{restatable}[Barrier for Continual learning with non-linear feature]{theorem}{lowerThm}
\label{thm:lower}
Let $k, r \geq 2, d\geq 3$. There exists a class of non-linear feature mappings and a sequence of environments, such that there is no (proper) continual learning algorithm that can guarantee to achieve less than $\frac{1}{1000}$-error over all environments with probability at least $1/2$.
\end{restatable}


%% file: relate.tex
\section{Related work}
\label{sec:related-work}

\paragraph{Continual learning in practice}
The study of continual learning (or lifelong learning) dates back to the work of \cite{thrun1995lifelong} and it receives a surge of research interest over recent years \citep{goodfellow2013empirical, kirkpatrick2017overcoming, farajtabar2020orthogonal, chaudhry2020continual, gepperth2016bio, rolnick2019experience, rusu2016progressive,javed2019meta,ramasesh2020anatomy,van2020brain}.
A central challenge in the field is to avoid {\em catastrophic forgetting} \citep{mccloskey1989catastrophic,mcclelland1995there}, which the work of \cite{goodfellow2013empirical} observed happened for gradient-based training of neural networks.
While there is a large amount of empirical work, we'll briefly summarize the dominant approaches. (We also refer the interested reader to more complete surveys \cite{parisi2019continual, chen2018lifelong}.) 
The {\em regularization based approach} alleviates catastrophic forgetting by posing constraints on the update of the neural weights. The elastic weight consolidation (EWC) approach \citep{kirkpatrick2017overcoming} adds weighted $\ell_2$ regularization to the objective function that penalizes the movement of neural weights.
The \emph{orthogonal gradient descent} (OGD) algorithm from \citep{farajtabar2020orthogonal, chaudhry2020continual} enforces the gradient update being orthogonal to update direction (by viewing the gradients as a high dimensional vector).
The {\em memory replay approach} restores data from previous tasks and alleviates catastrophic forgetting by rehearsing in the later tasks. \cite{rolnick2019experience} introduces experience replay to continual learning. \cite{gepperth2016bio} trains a deep generative model (a.k.a. GAN) to simulate past dataset for future use.
The {\em dynamic architecture approach} dynamically adjusts the neural network architecture to incorporate new knowledge and avoid forgetting. 
The progressive neural network \citep{rusu2016progressive} blocks changes
to the existing network and expands the architecture by allocating new a new subnet to be trained with the new information.

\paragraph{Continual learning in theory} In comparison to the vast empirical literature, theoretical works are comparatively few.
The works \citep{ruvolo2013ella, pentina2016lifelong,balcan2015efficient, cao2021provable} provide theoretical guarantees on lifelong learning. Their approaches can be categorized roughly into the {\em duplicate and fine-tuning} paradigm: The algorithm maintains a weighted combination over a family of representation functions and the focus is on the sample complexity guarantee. 
By contrast, we focus on the {\em feature extraction} paradigm and learn linear prompts on top of a {\em single} representation function. 
Both the duplicate-and-fine-tuning and the feature extraction paradigm have been extensively investigated in the literature, detailed discussions can be found at \cite{chen2018lifelong} and we provide a brief comparison.
From an algorithmic perspective, learning a weighted combination over a family of representation functions (i.e. the duplicate and fine-tuning) is much easier, as one can always initiates a new representation function for a new task. The algorithmic convenience allows previous literature focus more on the generalization and sample complexity guarantee, culminating with the recent work of \cite{cao2021provable}.
We note again that learning a single representation function and task specific linear prompts is much more challenging, but has practical benefits, e.g. memory efficiency. 
For example, in the applications of NLP, the basic representation function (e.g. BERT \citep{devlin2019bert}) is already overparameterized and usually contains billions of parameters. It is then formidable to maintain a large amount of the basic models and learn a linear combination over them.

\paragraph{Representation learning} More broadly, our work is also closely related to representation learning. Some recent theoretical works \cite{maurer2009transfer, maurer2013sparse, pontil2013excess,tripuraneni2020theory,maurer2016benefit, tripuraneni2021provable,du2020few} provide generalization and sample complexity guarantees for certain formalizations of multi-task learning based on the existence of a good representation. The work of \cite{rosenfeld2020risks,rosenfeld2021online} formulate the problem of out-of-distribution generalization and provide theoretical guarantee, similarly, under the assumption of a good representation.

%% file: algorithm.tex
\section{Continual learning with linear feature}
\label{sec:upper}
We restate our main result for linear feature mapping.

\linearThm*

\subsection{Algorithm}
\label{sec:algo}

A complete and formal description of $\DPGD$ is presented in Algorithm \ref{algo:ogd}. 
$\DPGD$ simultaneously updates the matrix of features $U$, as well as the linear classifier $v_i$ using gradient descent---with the restriction that the update of $U$ only occurs along directions that are orthogonal to the column span and row span of the previous feature matrix. 
Intuitively, one wishes the projection guarantees that existing features that have been learned are not erased or interfered by the new environment. Due to the quadratic nature of the loss, and the appearance of ``cross-terms'', 
this turns out to require both column and row orthogonality, and interestingly deviates from the practically common OGD method \citep{farajtabar2020orthogonal, chaudhry2020continual}.


In more detail, at the beginning of the $i$-th  ($i \in [k]$) environment, $\DPGD$ adds Gaussian noise to the feature matrix $U$ and the linear classifier $v_i$, to generate a good initialization for $U$ and $v_i$. 
Subsequently, we perform gradient descent to both the feature mapping matrix $U$ and linear classifier $v_i$---except $U$ is only updated along orthogonal directions w.r.t. the column span and the row span.
At the end of each environment, $\DPGD$ has a post-processing step to recover the ground truth $w_i$ by rounding each entry of $Uv_i$ to the nearest multiple of $\nu$,\footnote{This is the only place where we use the third regularity assumption. The rounding step could be removed if one runs gradient descent for long enough time. As the error $\|Uv_i -w_i\|_2$ decreases exponentially, one can directly use $Uv_i$ as a proxy. The exact recovery allows us to simplify the proof somewhat.} 
and then update the column and row span if the orthogonal component is non-negligible. The reason for the later step is that we only need to preserve row space when encountering new features.

\begin{algorithm}[!htbp]
\caption{Doubly projected gradient descent ($\DPGD$)}
\label{algo:ogd}
$\W \leftarrow \emptyset, \V\leftarrow \emptyset, U \leftarrow \mathbf{0}$  \Comment{$U \in \R^{d \times r}$}\\
$\sigma \leftarrow \tilde{O}(\frac{\eps}{d^2kD^4})$, $\eta \leftarrow O(\frac{\sigma^3}{k^2D^5})$, $T \leftarrow O(\frac{D}{\eta}\log\frac{Dkd}{\eps\nu}) + O(\frac{D}{\eta}\log \frac{k}{\sigma})$ \\
\For{$i =1, \ldots, k$}{
$U_{\init}\leftarrow \sigma \cdot P_{\W_{\perp}}\mathsf{rand}(d, r)P_{\V_{\perp}}$, $v_i \leftarrow \sigma\cdot \mathsf{rand}(r)$  \label{line:init1} \Comment{$U_{\init} \in \R^{d \times r}$, $v_i \in \R^{r}$}\\
$U \leftarrow U + U_{\init}$  \label{line:init2}\\
\For{$t =1, \ldots, T$}{
$\nabla_{U} \leftarrow \E_{(x, y) \sim \D_i}[x(x^{\top}Uv_i - y)v_i^{\top}], \nabla_{v_i} \leftarrow \E_{(x, y)\sim \D_i}[U^{\top}x(x^{\top}Uv_i - y)]$ \label{line:grad}\\
$U = U - \eta P_{\W_{\perp}}\nabla_{U} P_{\V_{\perp}}$ \label{line:o-update}\\
$v_i = v_i - \eta\nabla_{v_i}$
}
$\hat{w}_i \leftarrow \mathsf{Round}_{\nu}(Uv_i)$ \Comment{Round to the nearest multiple of $\nu$, $\hat{w}_i \in \R^{d}$} \label{line:exact}\\
\textbf{if} $\|P_{\W_{\perp}}\hat{w}_i\|_2 \geq 1/D$ \textbf{then} $\W \leftarrow \mathsf{span}(\W \cup \hat{w}_i)$, $\V \leftarrow \mathsf{span}(\V \cup v_i)$ \label{line:update}\\
$U \leftarrow P_{\W}U P_{\V}$\label{line:proj}
}
\end{algorithm}

\vspace{+2mm}
{\noindent \bf Parameters \ \ } We use $\sigma$ to denote the initialization scale, $\eta$ to denote the learning rate, and $T$ to denote the number of iterations for each task. These are all polynomially small parameters, whose scaling is roughly $D, d, k \ll \sigma^{-1} \ll \eta^{-1} < T$. 

\vspace{+2mm}
{\noindent \bf Notation \ \ } 
We write $[n] = \{1, 2, \ldots, n\}$, $[n_1:n_2] = \{n_1, \ldots, n_2\}$.
We use $\rand(n_1, n_2) \in \R^{n_1 \times n_2}$ to denote a size $n_1 \times n_2$ matrix whose entries are draw from random Gaussian $\mathsf{N}(0, 1)$. For each $i \in [k]$, $t \in [0:T]$, denote $U_{i, t}$ to be the feature matrix in the $t$-th iteration of the $i$-th environment (after performing the gradient update), denote $v_{i, t}$ similarly. 
$\DPGD$ includes a projection step at the end of $i$-th environment, we use $U_{i, \mathsf{end}}$ to denote the feature matrix after this projection. 
We use $\W_i$ (resp. $\V_i$) to denote the column (resp. row) space maintained at the end of $i$-th environment.
Let $\W_{\perp} \subseteq \R^{n}$ be the subspace orthogonal to $\W$ and define $\V_{\perp}$ similarly.
Let $P_\W$, $P_{\V}$, $P_{\W_{\perp}}$, $P_{\V_{\perp}}$ be the projection onto $\W$, $\V$, $\W_{\perp}$, $\V_{\perp}$ separately.

%% file: analysis.tex
\subsection{Analysis}
\label{sec:analysis}

We sketch the analysis of $\DPGD$ and prove Theorem \ref{thm:linear}. Due to space limitation, the detailed proof is deferred to Appendix \ref{sec:upper-app}.
The proof proceeds in the following four steps:
\begin{enumerate}
\item The first step, presented in Section \ref{sec:reduction}, reduces continual learning to a problem of continual matrix factorization and it allows us to focus on a more algebraically friendly objective function. 
\item  We then present some basic linear-algebraic facts to decompose the feature mapping matrix $U$, its gradient, and the loss into orthogonal components. 
The orthogonality of gradient update allows us to decouple the process of {\em leveraging the existing features} and the process of {\em learning a new feature},  as reflected in the loss terms and gradient update rules. See Section \ref{sec:decomposition} for details.
\item In Section \ref{sec:conv}, we zoom into one single environment, and prove $\DPGD$ provably converges to a global optimum, assuming the feature matrix $U$ from previous environment is well conditioned. This step contains the major bulk of our analysis: The objective function of continual matrix factorization is non-convex, and no regularization or spectral initialization used. (We cannot re-initialize, lest we destroy progress from prior environments.) 
\item Finally, in Section \ref{sec:induction}, we inductively prove that $\DPGD$ converges and the feature matrix is always well-conditioned. This wraps up the entire proof.
\end{enumerate}

\subsubsection{Reduction}
\label{sec:reduction}

We first recall the formal statement of the problem of continual matrix factorization.
\begin{definition}[Continual matrix factorization]
\label{def:prob_online}
Let $d, k, r \in \mathbb{N}$, $r \ll d, k$, $\eps > 0$. Let $W = [w_1, \ldots, w_k] = U^{\star} (V^{\star})^{\top} \in \R^{d \times k}$, where $U^{\star}\in \R^{d\times r}, V^{\star} \in \R^{k \times r}$. 
In an continual matrix factorization problem, the algorithm receives $w_i \in \R^{d}$ in the $i$-th step, and it is required to maintain a matrix $U \in \R^{d\times r}$ and output a vector $v_i \in \R^{r}$ such that
\begin{align}
\hat{L}(U, v_i) = \frac{1}{2}\|Uv_i - w_i\|_2^2 \leq \eps, \label{eq:prob-online}
\end{align}
and 
\begin{align}
\hat{L}(U, v_j) = \frac{1}{2}\|Uv_{j} - w_j\|_2^2 \leq \eps \quad j = 1, \ldots, i -1. \label{eq:prob-online1}
\end{align}
\end{definition}

The key observation is that running $\DPGD$ on the original continual learning objective~\eqref{eq:acc1} implicitly optimizes the continual matrix factorization objective \eqref{eq:prob-online} (Lemma \ref{lem:equiv1}).
Moreover, an $\eps$-approximate solution of continual matrix factorization is also an $\eps$-approximate solution of continual learning (Lemma \ref{lem:equiv2}).

\begin{lemma}[Gradient equivalence]
\label{lem:equiv1}
Under Assumption \ref{asp:distribution}, for any $i \in [k]$, the gradient update of objective~\eqref{eq:acc1} equals the gradient update of objective~\eqref{eq:prob-online}.
\end{lemma}

\begin{lemma}[Objective equivalence]
\label{lem:equiv2}
For any $w_1, \ldots, w_k \in \R^{d}$, $U \in \R^{d \times r}$ and $v_1, \ldots, v_k \in \R^{r}$, suppose $\hat{L}(U, v_i) = \tfrac{1}{2}\|Uv_i - w_i\|_2^2 \leq \eps$ holds for all $i \in [k]$, then $L(U, v_i) = \frac{1}{2}\E_{(x, y) \sim \D_i}(x^{\top}U v_i - y)^2 \leq \eps$.
\end{lemma}

Combining the above observations, it suffices to analyse $\DPGD$ for continual matrix factorization and prove Eq.~\eqref{eq:prob-online} and Eq.~\eqref{eq:prob-online1}.


\subsubsection{Decomposition}
\label{sec:decomposition}

We first provide some basic linear algebraic facts about orthogonal decompositions.
For any $i \in [k]$, we decompose $U_i, v_i, w_i$ along $\W_{i-1}$, $\W_{i-1,\perp}$, $\V_{i-1}$ and $\V_{i-1,\perp}$.

Let $w_i = w_{i, A} + w_{i, B}$ where $w_{i, A} \in \W_{i-1}$ and $w_{i, B} \in \W_{i-1, \perp}$. Note this decomposition is unique.
We focus on the case that $\|w_{i, B}\|_2 \in [1/D, D]$ in the following statements, and the case of $\|w_{i, B}\|_2 = 0$ carries over easily. (These are the only two cases, per Assumption \ref{asp:signal}).  
Similarly, let $U_{i} = U_{i, A} + U_{i, B}$, where each column of $U_{i, A}$ lies $\W_{i-1}$ and each column of $w_{i, B}$ lies in $\W_{i-1, \perp}$. (Note, again, $U_{i, A}$ and $U_{i,B}$ are unique.)
We further write $U_{i, B} = w_{i, B}x_i^{\top} + U_{i, 2}$, where the columns of $U_{i, 2}$ lie in $\W_{i-1, \perp} \backslash \{w_{i, B}\}$.
Finally, denote $v_{i} = v_{i,1} + v_{i, 2}$ with $v_{i, 1} \in \V_{i-1}$ and $v_{i, 2} \in \V_{i-1,\perp}$.

We summarize the decompositions mentioned above, with a few additional observations, in the lemma below: 

\begin{lemma}[Orthogonal decomposition] 
\label{lem:decomp}
For any $i \in [k]$ and any $t \in [0:T]$, there exists an unique decomposition of $U_{i, t}, w_i$ and $v_{i,t}$ of the form
\begin{align*}
    U_{i,t} = &~ U_{i, A, 0} + U_{i, B, t},  &~\column(U_{i, A, 0}) \in \W_{i-1}, \column(U_{i, B, t}) \in \W_{i-1, \perp},\\
    &&~\row(U_{i, A, 0}) \in \V_{i-1}, \row(U_{i, B, t}) \in \V_{i-1,\perp} \\
    w_i = &~ w_{i, A} + w_{i, B},  &~w_{i, A}\in  \W_{i-1}, w_{i, B} \in \W_{i-1, \perp}\\
    U_{i, B, t} =&~  w_{i, B}x_{i, t}^{\top} + U_{i, 2, t},  &~ x_{i, t} \in \V_{i-1, \perp}, \row(U_{i, 2, t}) \in \V_{i-1, \perp},  w_{i, B} \perp \column(U_{i, 2, t}) \\
    v_{i, t} = &~ v_{i, 1, t} + v_{i, 2, t}   &~ v_{i, 1, t} \in \V_{i-1}, v_{i, 2, t} \in \V_{i-1,\perp}.
\end{align*}
Here we use $\column(A), \row(A)$ to denote the column and row space of matrix $A$, and $\column(A)\in \W$ if the column space of $A$ is a subspace of $\W$. 
\end{lemma}

Since $U_{i, A, t}$ remains unchanged for $t = [0:T]$, we abbreviate it as $U_{i, A}$ hereafter.
We next provide the exact gradient update of each component under loss function $\hat{L}(U_i, v_i) = \frac{1}{2}\|U_iv_i - w_i\|_2^2$ and orthogonal projection.
\begin{lemma}[Gradient formula]
\label{lem:grad}
For any $i \in [k]$, the gradient update (after projection) obeys the relations: 
\begin{align*}
    \nabla_{x_i}(\hat{L}) = &~ v_{i, 2} (x^{\top}_{i}v_{i, 2} - 1) \\
    \nabla_{U_{2,i}}(\hat{L}) = &~ U_{i, 2}v_{i, 2}v_{i, 2}^{\top}\\
    \nabla_{v_{i,1}}(\hat{L}) = &~ U_{i, A}^{\top}U_{i, A}v_{i,1} - U_{i, A}^{\top} w_{i, A} \\ 
    \nabla_{v_{i,2}}(\hat{L}) = &~ \|w_{i, B}\|_2^2(x_i^{\top}v_{i, 2} - 1)x_{i} + U_{i, 2}^{\top}U_{i, 2}v_{i, 2}.
\end{align*}
\end{lemma}

We perform a similar decomposition to the loss function.
\begin{lemma}[Loss formula]
\label{lem:loss-formula}
For any $i \in [k], t\in [T]$, we have
\begin{align}
    \hat{L}(U_{i, t}, v_{i, t}) =  \frac{1}{2}\|U_{i, A}v_{i, 1, t} - w_{i, A}\|_2^2 + \frac{1}{2}\| w_{i, B}\|_2^2 (x_{i, t}^{\top}v_{i, 2, t} - 1)^2 + \frac{1}{2}\|U_{i, 2, t}v_{i, 2, t}\|_2^2. \label{eq:loss-decomp}
\end{align}
\end{lemma}

\vspace{+2mm}
{\noindent \bf Decoupling existing features from ``new'' features \ \ } We now offer some intuitive explanation for the decomposition. 
The first loss term in Eq.~\eqref{eq:loss-decomp} quantifies the error with already learned features. That is, the matrix $U_{i, A}$ stores existing features that have been learned, and it remains unchanged during the execution of the $i$-th environment; the component $v_{i, 1, t}$ is placed on top of $U_{i, A}$ and goal is to optimize $v_{i, 1,t}$ such that $U_{i, A}v_{i, 1, t}$ matches the component of $w_{i, A}$.
The second and last loss term quantify the loss on a new feature.
The valuable part of $U_{i,B}$ is $w_{i,B}x_{i, t}^{\top}$, where $w_{i, B}$ is the new feature component, and the matrix $U_{i, 2, t}$ can be thought of as random noise.
The vector $v_{i, 2, t}$ is placed on top of $w_{i, B} x^{\top}_{i,t}$, and intuitively, one should hope $x_{i, t}^{\top}v_{i, 2, t} = 1$ and this matches the new component of $w_{i, B}$.
At the same time, one hopes $U_{i,2,t}$ would disappear, or at least, $\|U_{i, 2,t}v_{i, 2,t}\|_2 \rightarrow 0$ when $t \rightarrow \infty$.

\subsubsection{Convergence}
\label{sec:conv}

For a fixed environment, we prove w.h.p. $\DPGD$ converges and the loss approaches to zero, given the initial feature mapping matrix $U_{i, A}$ is well conditioned.

\begin{lemma}
\label{lem:convergence}
For any $i\in [k]$, suppose $U_{i, A}$ satisfies $\frac{1}{2\sqrt{D}} \leq \sigma_{\min}(U_{i,A}) \leq \sigma_{\max}(U_{i,A}) \leq 2\sqrt{D}$, where $\sigma_{\min}(U_{i,A})$ and $\sigma_{\max}(U_{i, A})$ denote the minimum and maximum non-zero singular value of matrix $U_{i, A}$. After $T = O(\frac{D}{\eta}\log\frac{Dkd}{\eps\nu}) + O(\frac{D}{\eta}\log \frac{k}{\sigma})$ iterations, with probability at least $1 - O(1/k)$, the loss $\hat{L}(U_i,v_i) \leq \eps\nu/Dnk$.
\end{lemma}

{\noindent \bf Outline of the proof \ \ }
$\DPGD$ ensures existing features are preserved and it only optimizes the linear classifier, hence a linear convergence rate can be easily derived for the first loss term, given the feature matrix is well-conditioned (Lemma \ref{lem:fast}).
The key part is controlling the terms that capture learning with new features, i.e., the second and last loss term, where both the feature mapping $U_{i, B}$ and linear prompt $v_{i}$ get updated. 
In this case, the objective is non-convex and non-smooth. 
Existing works on matrix factorization or matrix sensing either require some fine-grained initialization (e.g. spectral initialization \citep{chi2019nonconvex}) or adding a regularization term that enforces smoothness \citep{ge2017no}, none of which are applicable in our setting. 
Our analysis draws inspiration from the recent work of \cite{ye2021global}, and divides the optimization process into two stages. We prove $\DPGD$ first approaches to a nice initialization position with high probability, and then show linear convergence. 

To be concrete, in the first stage, we prove (1) $x_{i, t}^{\top}v_{i, 2, t}$ moves closer to $1$, and (2) $\|x_{i, t} - \|w_{i, B}\|_2 v_{i, 2, t} \|_2 \approx 0$ (Lemma \ref{lem:stage-one}).
That is, the second loss term of Eq.~\eqref{eq:loss-decomp} decreases to a small constant while the pairs $x_{i, t}, v_{i, 2, t}$ remain balanced and roughly equal up to scaling. 
Meanwhile, we note that $U_{i, 2, t}$ is non-increasing, though the last loss term could still increase because $\|v_{i, 2, t}\|_2$ increases. 
In the second stage, we prove by induction that $\|U_{i, 2, t}^{\top}v_{i, 2, t}\|_2$ and $|x_{i, t}^{\top}v_{i, 2, t}-1|$ decay with a linear rate (hence converging to a global optimal), and $\|x_{i, t} - \|w_{i, B}\|_2 v_{i, 2, t} \|_2 \approx 0$ (Lemma \ref{lem:stage_two}).

First, we prove linear convergence for the first loss term.
\begin{lemma}[Fast learning on existing features]
\label{lem:fast}
For any $i \in [k]$ and $t\in [T]$, we have 
$$
\|U_{i,A}v_{i, 1,t} - w_{i, A}\|_2 \leq \left(1 - \frac{\eta}{4D}\right)^t\|U_{i, A}v_{i, 1, 0} -w_{i, A}\|_2.
$$
\end{lemma}

We next focus on the second and last loss terms. One can show that $x_{i,t}^{\top}v_{i,2, t}$ moves to $1$ while $\|\|w_{i, B}\|_2 x_{i, t} - v_{i, 2, t}\|_2$ remains small in the first  $T_1 = O(\frac{D}{\eta}\log \frac{k}{\sigma})$ iterations.
\begin{lemma}
\label{lem:stage-one}
With probability at least $1 - O(1/k)$ over the random initialization, there exists $T_1 = O(\frac{D}{\eta}\log \frac{k}{\sigma})$, such that for any $t \leq T_1$, one has
\begin{enumerate}
    \item $\|\|w_{i, B}\|_2 x_{i, t} - v_{i, 2, t}\|_2 \leq O(r\sigma \log(k/\sigma))$,
    \item $ x^{\top}_{i, t}v_{i, 2, t} < 0.9$ when $t < T_1$ and $0.9 <  x^{\top}_{i, T_1}v_{i, 2, T_1} < 1$,
    \item $U_{i, 2,t}^{\top}U_{i, 2, t} \preceq U_{i, 2, 0}^{\top}U_{i, 2, 0}$.
\end{enumerate}
\end{lemma}

A linear convergence of the second and the last loss terms can be shown, after the first $T_1$ iterations. 
Formally, we have:
\begin{lemma}
\label{lem:stage_two}
Let $T_2 = O(\frac{D}{\eta}\log(\frac{kdD}{\eps\nu}))$.
After $T = T_1 + T_2$ iterations, we have 
\begin{enumerate}
    \item $|x_{i, T}^{\top}v_{i, 2,T} - 1|\leq \eps\nu/kdD$,
    \item $\|U_{i, 2, T}^{\top}U_{i, 2,T}^{\top}v_{i, 2, T}\|_2 \leq \eps\nu$.
\end{enumerate}
\end{lemma}

Combining Lemma \ref{lem:loss-formula}, Lemma \ref{lem:fast}, Lemma \ref{lem:stage-one} and Lemma \ref{lem:stage_two}, one can conclude the proof of Lemma \ref{lem:convergence}.

\subsubsection{Induction}
\label{sec:induction}

Lemma \ref{lem:convergence} proves rapid convergence of $\DPGD$ for one single environment. To extend the argument to the whole sequence of environments, we need to ensure (1) the feature matrix is always well-conditioned and (2) catastrophic forgetting does not happen.
For (1), we need to analyse the limiting point of $\DPGD$ (there are infinitely many optimal solutions to Eq.~\eqref{eq:prob-online}), make sure it is well-balance and orthogonal to previous row/column space.
For (2), we make use of the orthogonality of $\DPGD$.
\begin{proof}[Proof Sketch of Theorem \ref{thm:linear}]
Due to the reduction established in Section \ref{sec:reduction}, it suffices to prove Eq.~\eqref{eq:prob-online} and Eq.~\eqref{eq:prob-online1}.
For each environment $i$ ($i \in [k]$), we inductively prove
\begin{enumerate}
    \item $\DPGD$ achieves good accuracy on the current environment, i.e., $\|U_{i, T}v_i - w_i\|_2 \leq \eps \nu$;
    \item The feature mapping matrix $U_{i}$ remains well conditioned, i.e. $\frac{1}{2\sqrt{D}} \leq \sigma_{\min}(U_{i, \mathsf{end}}) \leq\sigma_{\max}(U_{i, \mathsf{end}}) \leq 2\sqrt{D}$.
    \item The algorithm does not suffer from catastrophic forgetting, i.e., $\|U_{i, t}v_j - w_i\|_2 \leq \eps$ for any $j < i$ and $t\in [T]$.
\end{enumerate} 

The first claim is already implied by Lemma \ref{lem:convergence}. For the second claim, one first shows $\DPGD$ exactly recovers $w_i$ by taking $w_i = \hat{w}_i = \mathsf{Round}_{\nu}(U_{i, T}v_i)$. When $w_{i, B} = 0$, one can prove the feature matrix does not change, i.e, $U_{i, \mathsf{end}} = U_{i-1, \mathsf{end}}$; when $w_{i, B}\in [1/D, D]$, then one can show $U_{i, \mathsf{end}} \approx U_{i, \mathsf{end}} + \frac{1}{\|v_{i, 2, T}\|_2^2}w_{i, B}v_{i, 2, T}^{\top}$, as $w_{i, B} \perp \column(U_{i-1, \mathsf{end}}), v_{i, 2, T}\perp \row(U_{i-1, \mathsf{end}})$ and $ \|\frac{1}{\|v_{i, 2, T}\|_2^2}w_{B}v_{i, 2, T}^{\top}\|\leq O(\sqrt{D})$, the feature matrix $U$ remains well-conditioned.
The last claim can be derived from the orthogonality.
This wraps up the proof of Theorem \ref{thm:linear}.
\end{proof}

%% file: lower.tex
\section{Lower bound for non-linear features}
\label{sec:lower}

We next consider continual learning under a non-linear feature mapping. Learning with non-linear features turns out to be much more difficult, and our main result is to rule out the possibility of a (proper) continual learner. We restate the formal statement.

\lowerThm*

Our lower bound is constructed on a simple family of two-layer convolutional neural network with quadratic activation functions.
The input distribution is assumed to be uniform and the target function is a polynomial over the input.
The first environment is constructed such that multiple global optimum exist (hence the optimization task is under-constrained).
However, if a wrong optimum solution is picked, when the second environment is revealed, the non-linearity makes it impossible to switch back-and-forth.

\begin{proof}
It suffices to take $k = 2, n=3, d=2$. 
For both environments, we assume the input data are drawn uniformly at random from $\mathcal{B}_3(0, 1)$, where $\mathcal{B}_3(0,1)$ denotes the unit ball in $\R^3$ centered at origin.
The hypothesis class $\mathcal{H}$ consists of all two-layer convolutional neural network with a single kernel of size $2$ and the quadratic activation function. That is, the representation function is parameterized by $w \in \R^2$ and takes the form of $R_w(x) = (\langle w, x_{1:2}\rangle^2, \langle w, x_{2:3}\rangle^2) \in \R^2$, where $x\in \R^{3}$, $x_{i:j} \in \R^{j -i + 1}$ is a vector consists of the $i$-th entry to the $j$-th entry of $x$.  

The hard sequence of environments are drawn from the following distribution.
\begin{itemize}
    \item The objective function $f_1$ of the first environment is $f_1(x) = x_2^2$
    \item The objective function $f_2$ of the second environment equals $f_2(x) = x_3^2$ with probability $1/2$, and equals $f_2(x) = x_1^2$ with probability $1/2$.
\end{itemize}

First, the continual learning task is realizable: (1) if $f_2(x)= x_3^2$, then one can take $w = (0,1)$ and $v_1 = (1, 0), v_{2} = (0,1)$; (2) if $f_2(x) = x_1^2$, then one can take $w = (1, 0)$, $v_1 = (0,1)$, $v_2 = (1, 0)$.

We then prove no (proper) continual learning algorithm can guarantee to achieve less than $1/1000$-error on both environments with probability at least $1/2$.
Suppose the algorithm takes $v_1 = (v_{1,1}, v_{1,2})$ for the first environment. Due to symmetry, one can assume $|v_{1,1}| \geq |v_{1,2}|$.
With probability $1/2$, the objective function of the second environment is $f_2(x) = x_1^2$. Let $v_2 = (v_{2,1}, v_{2,2})$ be the linear prompt and $w = (w_{1}, w_{2})$ be the parameter of neural network. We prove by contradiction and assume
\begin{align*}
    \E_{x\sim \mathcal{B}_{3}(0,1)}[|\langle v_1, R_{w}(x)\rangle - x_2^2|^2] \leq 1/1000 ~~\text{and}~~ \E_{x\sim \mathcal{B}_{3}(0,1)}[|\langle v_2, R_{w}(x)\rangle - x_1^2|^2] \leq 1/1000.
\end{align*}
Let $\Pi_{n}^{d}$ be the space of all polynomial of degree at most $d$ in $n$ variables. By Lemma \ref{lem:tech}, notice that $\langle v_1, R_{w}(x)\rangle, \langle v_2, R_{w}(x)\rangle \in \Pi_{3}^2$, we must have that their coefficients match well with $x_2^2$ and $x_1^2$ respectively (in the sense that the absolute deviation is no larger than $1/4$).

First, compare the polynomials of $\langle v_2, R_{w}(x)\rangle$ and $x_1^2$, we must have (1) $v_{2,1}w_1^2 \geq 3/4$ due to the $x_1^2$ term, and due to the $x_1 x_2^2$ term, one has (2) $|v_{2,1}w_1w_2| \leq 1/4$. These two indicate (3) $|w_1| \geq 3 |w_2|$.
Then compare the polynomials of $\langle v_1, R_{w}(x)\rangle$ and $x_2^2$, we have (4) $|v_{1,1} w_1^2| \leq 1/4$ due to the $x_1^2$ term. Combining (3) and (4), one has (5) $|v_{1,1}w_2^2| \leq \frac{1}{9}|v_{1,1}w_1^2| \leq \frac{1}{36}$. Since the $x_2^2$ term is roughly matched, one must have (6) $|v_{1,2}w_1^2| \geq 1 - \frac{1}{4} - \frac{1}{36} = \frac{13}{18}$. However, note that (4) and (6) contradicts with the assumption that $|v_{1,1}| \geq |v_{1,2}|$. We conclude the proof.
\end{proof}


%% file: discussion.tex
\section{Conclusion}
\label{sec:discussion}
In this paper, we initiate a study of continual learning through {\em the feature extraction lens}, proposing an efficient gradient based algorithm, $\DPGD$, for the linear case, and a fundamental impossibility result in the general case. 
Our work leaves several interesting future directions. First, it would be interesting to generalize $\DPGD$ to non-linear feature mappings (perhaps even without provable guarantees) and conduct an empirical study of its performance. 
Second, our impossibility result does not rule out an improper continual learner, and in general, one can always maintain a task specific representation function and achieve good performance over all environments. It would be thus interesting to investigate what are the fundamental memory-accuracy trade-offs. Finally, our formulation assumes the task identity is known at test time; generalizing it to unknown task identity is direction for further work.

%% file: appendix-upper.tex
\section{Missing proof from Section \ref{sec:upper}}
\label{sec:upper-app}

\subsection{Missing proof from Section \ref{sec:reduction}}

We first present the proof of Lemma \ref{lem:equiv1}
\begin{proof}[Proof of Lemma \ref{lem:equiv1}]
For any $i \in [k]$, the gradient of feature matrix $U$ w.r.t. objective Eq.~\eqref{eq:acc1} equals
\begin{align}
\nabla_{U} = &~ \E_{(x, y) \sim \D_i}[x(x^{\top}Uv_i - y)v_i^{\top}] = \E_{x\sim \D_i}[x(x^{\top}Uv_i - x^{\top}w_i)v_i^{\top}]
= (Uv_i - w_i) v_i^{\top}.\label{eq:eqv1}
\end{align}
The first step follows from $y = x^{\top}w_i$ for any $(x, y)\sim \D_i$ and the second step follows from $\E_{x_i \sim \D_i}[xx^{\top}] = I_n$. The RHS of the above equation exactly equals the gradient of Eq.~\eqref{eq:prob-online} for $U$ (before and after projection to $\W_{i-1}$). 

We next observe
\begin{align}
\nabla_{v_i} = \E_{(x, y)\sim \D_i}[U^{\top}x(x^{\top}Uv_i - y)] = \E_{x\sim \D_i}[U^{\top}x(x^{\top}Uv - x^{\top}w_i)] =  U^{\top}(Uv_i - w_i), \label{eq:eqv2}
\end{align}
and the RHS of the above equation matches the gradient of Eq.~\eqref{eq:prob-online} for $v_i$. We conclude the proof here.
\end{proof}

We then include the proof of Lemma \ref{lem:equiv1}
\begin{proof}[Proof of Lemma \ref{lem:equiv2}]
We have
\begin{align*}
     \frac{1}{2}\E_{(x, y) \sim \D_i}(x^{\top}U v_i - y)^2 = &~ \frac{1}{2}\E_{(x, y) \sim \D_i}(x^{\top}U v_i - x^{\top}w_i)^2\\
     =&~ \frac{1}{2}(Uv_i - w)^{\top}\E_{x\sim\D_i}[xx^{\top}](Uv_i - w)^{\top} \\
     =&~ \frac{1}{2}\|Uv_i - w_i\|_2^2 \leq \eps.
\end{align*}
where the first step follows from $y = x^{\top}w_i$ for any $(x, y)\sim \D_i$ and the third step follows from $\E_{x_i \sim \D_i}[xx^{\top}] = I_n$. 
This concludes the proof.
\end{proof}

\subsection{Missing proof from Section \ref{sec:decomposition}}

We first present the proof of Lemma \ref{lem:decomp}
\begin{proof}[Proof of Lemma \ref{lem:decomp}]
For the first term, when $t=0$, one has $\column(U_{i, A, 0}) \in \W_{i-1}$ and $\column(U_{i, B, 0}) \in \W_{i-1,\perp}$, and these indicate (1) $U_{i, A, 0} = U_{i-1, \mathsf{end}}$, $\row(U_{i-1, \mathsf{end}}) \in \V_{i-1}$ and (2) $U_{i, B, 0} = U_{i, \init}$, $\row(U_{i, \init}) \in \V_{i-1, \perp}$. Hence we conclude $\row(U_{i, A, 0}) \in \V_{i-1}$ and $\row(U_{i, A, 0}) \in \V_{i-1, \perp}$. Since the gradient update is perform along $\W_{i-1, \perp}$ and $\V_{i-1, \perp}$, one has $U_{i, A}$ remains unchanged, i.e., $U_{i, A,t} = U_{i, A, 0}$ ($t\in [T]$), and the update of $U_{i, B, t}$ is along $V_{i-1, \perp}$, hence $\row(U_{i, B, t}) \in \V_{i-1, \perp}$ continues to hold.

For the third term, for any $t\in [0:T]$, one has
\[
\V_{i-1, \perp} \ni w_{i, B}^{\top} U_{i, B, t} = w_{i, B}^{\top}w_{i, B}x_{i, t}^{\top} + w_{i, B}^{\top}U_{i, 2, t} = \|w_{i, B}\|_2^2 x_{i,t}^{\top},
\]
where the second step follows from $\column(U_{i, 2,t}) \in \W_{i-1, \perp}\backslash \{w_{i, B} \}$. Hence we conclude $x_{i, t} \in V_{i-1, \perp}$. Since $\row(U_{i, B, t}), \row(w_{i, B} x_{i, t}^{\top}) \in \V_{i-1, \perp}$, one has $\row(U_{i, 2, t}) \in V_{i-1 , \perp}$.
\end{proof}

We then prove
\begin{proof}[Proof of Lemma \ref{lem:grad}]
The gradient of $U$ (before projection) satisfies
\begin{align*}
    \nabla_{U_{i}} = &~ (U_{i}v_i - w_i)v_{i}^{\top}\\ 
    =&~  U_{i, A} v_{i}v_{i}^{\top} + U_{i, B} v_{i}v_{i}^{\top} - w_{i, A} v_{i}^{\top} - w_{i, B} v_{i}^{\top} \\
    = &~ (U_{i, A} v_{i}v_{i}^{\top} - w_{i, A}v_{i}^{\top}) + (w_{i,B} x_{i}^{\top} + U_{i,2}) v_{i}v_{i}^{\top} - w_{i,B}v_{i}^{\top}\\
    = &~ (U_{i, A} v_{i}v_{i}^{\top} - w_{i, A}v_{i}^{\top}) + w_{i, B} v_{i}^{\top}(x_{i}^{\top}v_{i} - 1) + U_{i,2} v_{i}v_{i}^{\top},
\end{align*}
where the first step follows from Eq.~\eqref{eq:eqv1}, the second and third steps follow from the first three terms of Lemma \ref{lem:decomp}. 

The actual update (after projection) obeys
\begin{align*}
    P_{\W_{i-1,\perp}}\nabla_{U_{i}}P_{\V_{i-1,\perp}} = &~ P_{W_{i-1,\perp}} ((U_{i, A} v_{i}v_{i}^{\top} - w_{i,A}v_{i}^{\top}) + w_{i, B} v_{i}^{\top}(x_{i}^{\top}v_{i} - 1) + U_{i, 2} v_{i}v_{i}^{\top}) P_{\V_{i-1\perp}}\\
    = &~ (w_{i,B} v_{i}^{\top}(x_{i}^{\top}v_{i} - 1) + U_2 v_{i}v_{i}^{\top}) P_{\V_{i-1\perp}}\\
    = &~ w_{i, B}  v_{i,2}^{\top} (x_{i}^{\top}v_{i,2} - 1)   + U_{i,2} v_{i, 2}  v_{i,2}^{\top},
\end{align*}
where the second step follows from $w_{i,A}, \column(U_{i,A}) \in \W_{i-1}$, the third step follows from $\row (U_{i,2}) \in \V_{i-1,\perp}$ and $x_{i} \in \V_{i-1, \perp}$, see Lemma \ref{lem:decomp} for details.

Hence, we conclude
\begin{align*}
\nabla_{x_{i}} = &~ v_{i, 2} (x_i^{\top}v_{i,2} - 1) \quad \text{and} \quad \nabla_{U_{i,2}} = U_{i, 2}v_{i, 2}v_{i, 2}^{\top}.
\end{align*}

We next calculate the gradient of $v$, it satisfies
\begin{align*}
    \nabla_{v_i} =&~ U_{i}^{\top}(U_{i}v_{i} - w_i)\\
    =&~ U_{i,A}^{\top}U_{i,A}v_{i} + U_{i,B}^{\top}U_{i,B}v_{i} - U_{i,A}^{\top}w_{i,A} - U_{i,B}^{\top} w_{i,B}\\
    = &~ U_{i,A}^{\top}U_{i,A}v_{i,1} - U_{i,A}^{\top} w_{i,A} +  U_{i,B}^{\top}U_{i,B}v_{i,2} -  U_{i,B}^{\top}w_{i,B}.
\end{align*}
The first step follows from Eq.~\eqref{eq:eqv2}, the second step follows from the first two terms of Lemma \ref{lem:decomp}.
The third step uses the fact that $\mathsf{row}(U_{i,A}) \in \V_{i-1}$, $v_{i,1} \in \V_{i-1}$, $v_{i,2} \in \V_{i-1,\perp}$ and $\mathsf{row}(U_{i,B}) \in \V_{i-1,\perp}$

Hence, we have
\begin{align*}
\nabla_{v_{i,1}} = U_{i,A}^{\top}U_{i,A}v_{i,1} - U_{i,A}^{\top} w_{i,A} 
\end{align*}
and 
\begin{align*}
    \nabla_{v_{i,2}} = &~ U_{i,B}^{\top}U_{i,B}v_{i,2} -  U_{i,B}^{\top}w_{i,B} \\
    = &~ (w_{i,B} x_{i}^{\top} + U_{i,2})^{\top}(w_{i,B} x_{i}^{\top} + U_{i,2})v_{i,2} - (w_{i,B} x_{i}^{\top} + U_{i,2})^{\top}w_{i,B} \\
    = &~ x_i \|w_{i,B}\|_2^2 x_{i}^{\top}v_{i,2} + U_{i,2}^{\top}U_{i,2}v_{i,2} - x_{i} \|w_{i,B}\|_2^2 \\
    = &~ \|w_{i,B}\|_2^2(x_{i}^{\top}v_{i} - 1)x_i + U_{i,2}^{\top}U_{i,2}v_{i,2},
\end{align*}
where the third step holds due to $w_{i,B} \perp \column(U_{i,2})$. We conclude the proof here. 
\end{proof}

Finally, we prove
\begin{proof}[Proof of Lemma \ref{lem:loss-formula}]
For any $i \in [k], t \in[T]$, we have
\begin{align*}
    \|U_{i,t}v_{i,t} - w_{i}\|_2^2 = &~ \|(U_{i,A} + U_{i, B,t})(v_{i,1,t} + v_{i, 2, t}) - w_{i, A} - w_{i, B} \|_2^2\\ 
    = &~ \|U_{i, A}v_{i, 1,t} +  U_{i, B,t} v_{i, 2, t} - w_{i, A} - w_{i, B} \|_2^2\\
    = &~\|U_{i, A, t}v_{i,1, t} - w_{i, A}\|_2^2 + \| U_{i, B, t} v_{i, 2, t} - w_{i, B}\|_2^2\\
    = &~  \|U_{i,A}v_{i,1, t} - w_{i, A}\|_2^2 + \| (w_{i,B}x_{i, t}^{\top} + U_{i, 2, t}) v_{i, 2, t} - w_{i, B}\|_2^2 \\
    = &~\|U_{i,A}v_{i, 1, t} - w_{i,A}\|_2^2 + \| w_{i, B}\|_2^2 (x_{i, t}^{\top}v_{i, 2,t} - 1)^2 + \|U_{i, 2, t}v_{i, 2, t}\|_2^2.
\end{align*}
The second step follows from $\row(U_{i,A}) \in \V_{i-1}$, $\row(U_{i,B}) \in \V_{i-1, \perp}$, $v_{i, 1, t} \in \V_{i-1}$, $v_{i, 2, t} \in \V_{i-1,\perp}$, the third step follows from $U_{i,A}v_{i, 1, t} - w_{i, A} \in \W_{i-1}$ and $U_{i, B, t}v_{2, t} - w_{i, B} \in \W_{i-1,\perp}$. The last step follows from $w_{i, B} \perp \column(U_{i, 2, t})$.
\end{proof}

\subsection{Missing proof from Section \ref{sec:conv}}

In the proof, we write $x = y \pm z$ if $x \in [y - z, y+z]$. For simplicity, we assume $\log(1/\eps\nu) \ll k, d$.

\begin{proof}[Proof of Lemma \ref{lem:fast}]
This follows easily from the standard analysis of gradient descent for least square regressions. For any $t \in [0: T-1]$, one has
\begin{align*}
    \|U_{i, A}v_{i, 1,t+1} - w_{i, A}\|_2 = &~ \|U_{i, A} (v_{i, 1,t} - \eta (U_{i, A}^{\top}U_{i, A}v_{i, 1,t} - U_{i, A}^{\top}w_{i,A})) - w_{i, A}\|_2\\
    = &~ \|(I - \eta U_{i,A}U_{i, A}^{\top})(U_{i,A}v_{i, 1,t} - w_{i, A})\|_2 \\
    \leq &~ (1- \frac{\eta}{4D}) \|(U_{i, A}v_{i, 1,t} - w_{i, A})\|_2.
\end{align*}

The first step follows from the gradient update formula (see Lemma \ref{lem:grad}), the third step follows from $U_{i, A}v_{i, 1, t} - w_{i, A} \in \column(U_{i, A})$, and $2\sqrt{D} \geq \sigma_{\max}(U_{i, A}) \geq \sigma_{\min}(U_{i, A}) \geq \frac{1}{2\sqrt{D}}$ and $\eta < \frac{1}{4D}$. We conclude the proof here.
\end{proof}

We next proceed to the proof of Lemma \ref{lem:stage-one}
\begin{proof}[Proof of Lemma \ref{lem:stage-one}]
Recall our goal is to prove
\begin{enumerate}
    \item $\|\|w_{i, B}\|_2 x_{i, t} - v_{i, 2, t}\|_2 \leq O(r\sigma \log(k/\sigma))$,
    \item $ x_{i, t}^{\top}v_{i, 2, t} < 0.9$ when $t < T_1$ and $0.9 <  x^{\top}_{i, T_1}v_{i, 2, T_1} < 1$,
    \item $U_{i, 2,t}^{\top}U_{i, 2, t} \preceq U_{i, 2, 0}^{\top}U_{i, 2, 0}$.
\end{enumerate}

We inductively prove these three claims. For the base case, we have that 
\begin{align}
x_{i, 0} = &~ \frac{1}{\|w_{i, B}\|_2^2}w_{i, B}^{\top}U_{i, B, 0} = \frac{1}{\|w_{i, B}\|_2^2}w_{i, B}^{\top}P_{\W_{i-1, \perp}}U_{i,  \init}P_{\V_{i-1, \perp}} =\frac{1}{\|w_{i, B}\|_2^2}w_{i, B}^{\top}U_{i,  \init}P_{\V_{i-1,\perp}} \notag\\
\approx &~ \frac{\sigma}{\|w_{i, B}\|_2}\cdot   \mathsf{rand}(r,1) P_{\V_{i-1\perp}}, \label{eq:x_init}
\end{align}
where in the first step we use the fact that $U_{i, B, 0} = w_{i, B} x_{i, 0}^{\top} + U_{i, 2, 0}$, $w_{i, B} \perp \column(U_{i, 2, 0})$, in the third step, we use $w_{i, B} \in \W_{i-1, \perp}$. The fourth step follows from $\frac{1}{\|w_{i, B}\|_2^2}w_{i, B}^{\top}U_{i, \init}$ is a random Gaussian vector with variance $\frac{\sigma}{\|w_{i, B}\|_2}$.
Similarly, we have
\begin{align}
    v_{i, 2, 0} = \sigma \cdot \mathsf{rand}(r,1) P_{\V_{i-1, \perp}}. \label{eq:v_init}
\end{align}
Hence, with probability at least $1 - O(1/k)$, we have
\begin{align*}
    \|\|w_{i, B}\|_2 x_{i, 0} - v_{i, 2, 0}\|_2 \leq O(r\sigma\log(k)) \quad \text{and} \quad x_{i, 0}^{\top}v_{i, 2,0} < O(\sigma^2 r D \log(k))\ll 1.
\end{align*}

We have proved the base case. Now suppose the induction holds up to time $t$, for the $(t+1)$-th iteration, 
we first go over the first claim. One has
\begin{align}
    &~ \|\|w_{i, B}\|_2x_{i, t + 1} - v_{i, 2, t + 1}\|_2^2 - \|\|w_{i, B}\|_2 x_{i, t} - v_{i, 2, t}\|_2^2 \notag\\
    = &~\|\|w_{i, B}\|_2(x_{i, t} - \eta v_{i, 2,t}(x_{i, t}^{\top}v_{i, 2,t} - 1)) - (v_{i, 2, t} -\eta x_{i, t} \|w_{i, B}\|_2^2(x_{i, t}^{\top}v_{i, 2,t} - 1) - \eta U_{i, 2,t}^{\top}U_{i, 2,t}v_{i, 2,t})\|_2^2 \notag\\
    &~ - \|\|w_{i, B}\|_2x_{i, t} - v_{i, 2, t}\|_2^2\notag\\
    = &~ \|(\|w_{i, B}\|_2 x_{i, t} - v_{i, 2,t}) - \eta v_{i, 2,t}\|w_{i, B}\|_2(x_{i, t}^{\top}v_{i, 2,t} - 1) + \eta x_{i, t} \|w_{i, B}\|_2^2(x_{i, t}^{\top}v_{i, 2,t} - 1) + \eta U_{i, 2,t}^{\top}U_{i, 2,t}v_{i, 2,t}\|_2^2 \notag \\
    &~ - \|\|w_{i, B}\|_2 x_{i, t} - v_{i, 2, t}\|_2^2 \notag\\
    = &~ 2\eta \langle \|w_{i, B}\|_2 x_{i, t} - v_{i, 2,t}, x_{i, t} \|w_{i,B}\|_2^2(x_{i, t}^{\top}v_{i, 2,t} - 1) - v_{i, 2}\|w_{i, B}\|_2 (x_{i, t}^{\top}v_{i, 2,t} - 1) + U_{i, 2,t}^{\top}U_{i, 2,t}v_{i, 2,t}\rangle\notag\\
    &~ \pm O(\eta^2 D^4) \notag\\
    = &~ 2\eta \|w_{i,B}\|_2(x_{i, t}^{\top}v_{i, 2,t} - 1) \|\|w_{i, B}\|_2x_{i, t} - v_{i, 2,t}\|_2^2 +\eta\langle \|w_{i, B}\|_2 x_{i, t} - v_{i, 2,t}, U_{i, 2,t}^{\top}U_{i, 2,t}v_{i, 2,t}\rangle\notag\\
    &~ \pm O(\eta^2 D^4) \label{eq:indc_mid}\\
    \leq &~ 2\eta\langle \|w_{i, B}\|_2 x_{i, t} - v_{i, 2,t}, U_{i, 2,t}^{\top}U_{i, 2,t}v_{i, 2,t}\rangle \pm O(\eta^2 D^4)\notag\\
    \leq &~ \tilde{O}(\eta r \sigma^3 d^2 D) + O(\eta^2 D^4). \label{eq:indc1}
\end{align}
The first step follows from the gradient update formula (see Lemma \ref{lem:grad}), the third step follows from that 
\begin{align*}
 &~\|U_{i, 2,t}^{\top}U_{i, 2,t} v_{i, 2,t}\|_2 \ll 1, \quad \|\|w_{i, B}\|_2^2(x_{i, t}^{\top}v_{i, 2,t} - 1) x_{i, t}\|_2 \leq O(D^2)\\
\end{align*}
and 
\begin{align*}
\|\|w_{i, B}\|_2(x_{i, t}^{\top}v_{i, 2,t} - 1)v_{i, 2,t}\|_2 \leq O(D^2),
\end{align*}
which can be derived easily from the induction hypothesis.
The fifth step follows from $x_{i, t}^{\top}v_{i, 2,t}  < 1$ when $t \leq T_1$.
The last step follows from 
\begin{align}
\|\|w_{i, B}\|_2x_{i, t} - v_{i, 2,t}\| \leq \tilde{O}(r\sigma),  \|U_{i, 2,t}^{\top}U_{i, 2,t}\| \leq \|U_{i, 2,0}^{\top}U_{i, 2,0}\| \leq  \tilde{O}(d^2\sigma^2), \|v_{i, 2,t}\|_2 \leq O(D), \label{eq:detail1}
\end{align}
which can be derived easily from the induction hypothesis.
Combining with $\eta \leq \frac{\sigma^2}{D^5}$, $\sigma \leq \frac{1}{D^2 d^2}$ and the total number of iteration is $T_1 \leq O(\frac{D}{\eta}\log \frac{k}{\sigma})$, one can proved the first claim.

For the second claim, we have that
\begin{align}
    &~x_{i, t+1}^{\top}v_{i, 2,t+1} - x_{i, t}^{\top}v_{i, 2,t}\notag \\
    = &~ (x_{i, t} - \eta v_{i, 2,t}(x_{i, t}^{\top}v_{i, 2,t} - 1))^{\top}(v_{i, 2,t} - \eta x_{i, t} \|w_{i, B}\|_2^2(x_{i, t}^{\top}v_{i, 2,t} - 1) - \eta U_{i, 2,t}^{\top}U_{i, 2,t} v_{i, 2,t}) - x_{i, t}^{\top}v_{i, 2,t}\notag \\
    = &~ - \eta (\|w_{i, B}\|_2^2 \|x_{i, t}\|_2^2 + \|v_{i, 2,t}\|_2^2)(x_{i, t}^{\top}v_{i, 2,t} - 1) - \eta x_{i, t}^{\top}U_{i, 2,t}^{\top}U_{i, 2, t} v_{i, 2,t} \pm O(\eta^2 D^3)  \label{eq:grad_xv} \\
    \geq &~ \frac{1}{2}\eta (\|w_{i, B}\|_2^2 \|x_{i, t}\|_2^2 + \|v_{i, 2,t}\|_2^2)(x_{i, t}^{\top}v_{i, 2,t} - 1) - O(\eta^2 D^3) \notag\\
    \geq &~ \frac{1}{20}\eta(\|w_{i, B}\|_2^2 \|x_{i, t}\|_2^2 + \|v_{i, 2,t}\|_2^2) - O(\eta^2D^3).\label{eq:indc2}
\end{align}
The first step follows from the gradient update formula (see Lemma \ref{lem:grad}), the second step holds since
\[
\|v_{i, 2,t}(x_{i, t}^{\top}v_{i, 2,t} - 1))\|_2 \leq O(D), \quad \| \|w_{i, B}\|_2^2(x_{i, t}^{\top}v_{i, 2,t} - 1)x_{i, t}\|_2 \leq O(D^2) \quad \text{and} \quad \|U_{i, 2,t}^{\top}U_{i, 2,t}v_{i, 2,t}\|_2 \ll 1.
\]
Again, these inequalities can be derived easily from the inductive hypothesis.
The third step holds since $U_{i, 2,t}^{\top}U_{i, 2,t} \preceq U_{i, 2,t}^{\top}U_{i, 2,t}  \preceq \tilde{O}(d^2\sigma^2)\cdot I$, and therefore,
\[
|x_{i, t}^{\top}U_{i, 2,t}^{\top}U_{i, 2, t} v_{2,t}| \leq \tilde{O}(d^2 \sigma^2) \cdot \|x_{i, t}\|_2\|v_{i, 2,t}\|_2  \ll |(\|w_{i, B}\|_2^2 \|x_{i, t}\|_2^2 + \|v_{i, 2,t}\|_2^2)(x_{i, t}^{\top}v_{i, 2,t} - 1) | .
\]
The last step uses the fact that $x_{i, t}^{\top}v_{2,t} < 0.9$ when $t < T_1$.

We next bound the RHS of Eq.~\eqref{eq:indc2} and prove it can not be too small. We focus on $\|\|w_{i, B}\|_2 x_{i, t+1} + v_{i, 2, t+1}\|_2$ and prove it monotonically increasing. In particular, at initialization, with probability at least $1 - O(1/k)$, due to anti-concentration of Gaussian, we have
\begin{align}
\|\|w_{i, B}\|x_{i, 0} + v_{i, 2,0}\|_2 \approx \sigma \|\mathsf{rand}(r,1) P_{\V_{i-1,\perp}}\|_2 \geq \sigma/k.
\end{align}
Furthermore, we have
\begin{align}
    &~ \|\|w_{i-1,B}\|_2 x_{i, t+1} + v_{i, 2, t+1}\|_2^2 \notag \\
    =&~ \|\|w_{i-1, B}\|_2 (x_{i, t} - \eta v_{i, 2,t}(x_{i, t}^{\top}v_{i, 2,t} - 1)) + (v_{i, 2, t} - \eta x_{i, t}\|w_{i, B}\|_2^2(x_{i, t}^{\top}v_{i, 2,t} - 1) - \eta U_{i, 2,t}^{\top}U_{i, 2,t} v_{i, 2,t}) \|_2^2 \notag \\
    = &~ \|\|w_{i, B}\|_2 x_{i, t} + v_{i, 2, t}\|_2^2 + 2\eta \|w_{i, B}\|_2(1- x_{i, t}^{\top}v_{i, 2,t})\| \|w_{i, B}\|_2 x_{i, t} + v_{i, 2, t}\|_2^2 \notag \\
    &~ + \eta\langle \|w_{i, B}\|_2x_{i, t} + v_{i, 2,t}, U_{i, 2,t}^{\top}U_{i, 2,t} v_{i, 2,t} \rangle \pm O(\eta^2D^4)\notag\\
    \geq &~ (1 + \frac{1}{20}\eta \|w_{i, B}\|_2) \|w_{i, B} x_{i, t} + v_{i, 2, t}\|_2^2, \label{eq:monotone}
\end{align}
where the first step holds due to the gradient update formula (see Lemma \ref{lem:grad}), the second step holds due to Eq.~\eqref{eq:detail1}. 
The last step holds since 
\begin{align*}
\|U_{i, 2,t}^{\top}U_{i, 2,t} v_{i, 2,t}\|_2 \leq &~ \tilde{O}(d^2\sigma^2 D) \ll \frac{\sigma}{40 kD} \leq \frac{1}{40}\|w_{i, B}\|_2\cdot \|\|w_{i, B}\|_2x_{i, 0} + v_{i, 2,0}\|_2\\
\leq &~ \frac{1}{40}\|w_{i, B}\|_2\cdot \|\|w_{i, B}\|_2x_{i, t} + v_{i, 2,t}\|_2
\end{align*}
and 
\[
O(\eta D^4) \ll \frac{\sigma^2}{40 k^2D} \leq \frac{1}{40}\|w_{i, B}\|_2\cdot \|\|w_{i, B}\|_2x_{i, 0} + v_{i, 2,0}\|_2^2 \leq \frac{1}{40}\|w_{i, B}\|_2\cdot \|\|w_{i, B}\|_2x_{i, t} + v_{i, 2,t}\|_2^2.
\]

Hence, we conclude that $\|\|w_{i, B}\|_2x_{t} + v_{2,t}\|_2$ is monotonically increasing, and in particular, 
\begin{align*}
\|\|w_{i, B}\|_2 x_{i, t} + v_{i, 2, t}\|_2^2 \geq &~  \|\|w_{i, B}\|_2 x_{i, 0} + v_{i, 2, 0}\|_2^2 = \Omega(\sigma^2/k^2)  &~ \forall t\in [T_1]\\
\|\|w_{i, B}\|_2 x_{i, t} + v_{i, 2, t}\|_2^2 \geq &~  \Omega(1)  &~ t \geq O(\frac{D}{\eta}\log\frac{k}{\sigma})
\end{align*}
The second inequality follows from Eq.~\eqref{eq:monotone}.
Plugging into Eq.~\eqref{eq:indc2}, one has 
\begin{align*}
    x_{i, t+1}^{\top}v_{i, 2,t+1} - x_{i, t}^{\top}v_{i, 2,t} \geq &~  \frac{1}{20}\eta(\|w_{i, B}\|_2^2 \|x_{i, t}\|_2^2 + \|v_{i, 2,t}\|_2^2) - O(\eta^2 D^3)\\
    \geq &~  \frac{1}{40}\eta(\|\|w_{i, B}\|_2 x_{i, t}+ v_{i, 2,t}\|_2^2) - O(\eta^2 D^3)\\
    \geq &~ \left\{
    \begin{matrix}
    0 & t \in [T]\\
    \Omega(\eta) & t \geq  O(\frac{D}{\eta}\log\frac{k}{\sigma})
    \end{matrix}
    \right.
\end{align*}

Hence, after at most $T_1 \leq  O(\frac{D}{\eta}\log\frac{k}{\sigma})$ iterations, we have $0.9 \leq x_{i, T_1}^{\top}v_{i, 2,T_1} < 1$. It would not exceed $0.9$ too much since by Eq.~\eqref{eq:grad_xv}, the change per iteration is at most
\begin{align}
|x_{i, t+1}^{\top}v_{i, 2,t+1} - x_{i, t}^{\top}v_{i, 2,t} | \lesssim &~ \eta (\|w_{i, B}\|_2^2 \|x_{i, t}\|_2^2 + \|v_{i, 2,t}\|_2^2)\notag\\ 
\leq &~ \eta (\|\|w_{i, B}\|_2x_{t} - v_{i, 2,t}\|_2^2 + 2\|w_{i, B}\|_2 x_{i, t}^{\top}v_{i, 2,t}) \leq 4\eta D  \ll 1 \label{eq:smooth}
\end{align}

For the third claim, we have 
\begin{align*}
    U_{i, 2,t+1}^{\top}U_{i, 2,t+1} = (U_{i, 2,t} - \eta U_{i, 2,t}v_{i, 2,t}v_{i, 2,t}^{\top})^{\top}(U_{i, 2,t} - \eta U_{i, 2,t}v_{i, 2,t}v_{i, 2,t}^{\top}) \preceq U_{i, 2,t}^{\top} U_{i, 2,t}.
\end{align*}
The last step holds since $(I - v_{i, 2,t}v_{i, 2,t}^{\top})$ is a PSD matrix and  $(I - v_{i, 2,t}v_{i, 2,t}^{\top}) \preceq I$. We have proved all three claims.
\end{proof}

\begin{proof}[Proof of Lemma \ref{lem:stage_two}]
For the $t$-th iteration ($t \in [T_1: T_2]$), we prove the following claims inductively.
\begin{enumerate}
    \item $|x_{i, t}^{\top}v_{i, 2,t}- 1| \leq \frac{1}{2}(1 - \frac{\eta}{4D})^{t - T_1}$,
    \item $\|U_{i, t}v_{i, t}\|_2 \leq (1 - \frac{\eta}{4D})^{t - T_1}$,
    \item $\|\|w_{i, B}\|_2 x_{i, t} - v_{i, 2, t}\|_2 \leq \tilde{O}(r\sigma$).
\end{enumerate}

The inductive base ($t= T_1$) holds trivially. Assuming the hypothesis holds up to time $t$, we start from the first claim. We have that
\begin{align*}
    &~ (1 - x_{i, t+1}^{\top}v_{i, 2,t+1}) - (1 - x_{i, t}^{\top}v_{i, 2,t})\\
    = &~ -(x_{i, t} - \eta v_{i, 2,t}(x_t^{\top}v_{i, 2,t} - 1))^{\top}(v_{i, t} - \eta x_{i, t}\|w_{i, B}\|_2^2(x_{i, t}^{\top}v_{i, 2,t} - 1) - \eta U_{i, 2,t}^{\top}U_{i, 2,t} v_{i, 2,t}) + x_{i, t}^{\top}v_{i, 2,t} \\
    = &~  \eta (\|w_{i, B}\|_2^2 \|x_{i, t}\|_2^2 + \|v_{i, 2,t}\|_2^2)(x_{i, t}^{\top}v_{i, 2,t} - 1) + \eta x_{i, t}^{\top}U_{i, 2,t}^{\top}U_{i, 2, t} v_{i, 2} \pm O(\eta^2 D^3 |x_{i, t}^{\top}v_{i, 2,t} - 1|).
\end{align*}
The first step follows from the gradient update formula (see Lemma \ref{lem:grad}), the second step follows from
\[
\|w_{i, B}\|_2^2 \|x_{i, t}\| \leq O(D^2), \quad \|v_{i, 2,t}\|_2 \leq D \quad \text{and} \quad \|U_{i, 2,t}v_{i, 2,t}\|_2 \ll 1.
\]
Since
\[
\|w_{i, B}\|_2^2 \|x_{i, t}\|_2^2 + \|v_{i, 2,t}\|_2^2 = \|\|w_{i, B}\|x_{i, t} - v_{i, 2,t}\|_2^2 + 2\langle \|w_{i, B}\|_2x_{i, t}, v_{i, 2,t}\rangle  \geq \frac{1}{D} \]
holds due to our inductive hypothesis, we further have that
\begin{align}
\label{eq:indc_4}
   |1 - x_{i, t+1}^{\top}v_{i, 2,t+1}| \leq (1 - \frac{\eta}{D}) |1 - x_{i, t}^{\top}v_{i, 2,t}| + \eta |x_{i, t}^{\top}U_{i, 2,t}^{\top}U_{i, 2, t} v_{i, 2, t}| \pm O(\eta^2D^3 |x_{i, t}^{\top}v_{i, 2,t} - 1|).
\end{align}

{\noindent \bf Case 1.  } Suppose $\frac{1}{2}(1- \frac{\eta}{4D})^{t+2 - T_1} \leq |x_{i, t}^{\top}v_{i, 2,t}- 1| \leq \frac{1}{2}(1- \frac{\eta}{4D})^{t - T_1}$, 
then we have 
\begin{align*}
|1 - x_{i, t+1}^{\top}v_{i, 2,t+1}| \leq (1 - \frac{\eta}{4D}) |1 - x_{i, t}^{\top}v_{i, 2,t}| \leq \frac{1}{2}(1 - \frac{\eta}{4D})^{t+1- T_1}.
\end{align*}
This holds due to Eq. \eqref{eq:indc_4}, $\eta D^3 \ll \frac{1}{4D}$ and
\begin{align*}
|x_{i, t}^{\top}U_{i, 2,t}^{\top}U_{i, 2, t} v_{i, 2,t}| \leq \|x_{i, t}^{\top}U_{i, 2,t}^{\top}\|_2 \|U_{i, 2, t} v_{i, 2,t} \|_2 \leq \tilde{O}(Dd\sigma) \cdot 2 |x_{i, t}^{\top}v_{i, 2,t} - 1| \leq \frac{1}{4D}|x_{i, t}^{\top}v_{i, 2,t} - 1|,
\end{align*}
where the second step holds due to the induction hypothesis.

{\noindent \bf Case 2. } Suppose $|x_{i, t}^{\top}v_{i, 2,t}- 1| \leq \frac{1}{2}(1- \frac{\eta}{4D})^{t+2 - T_1}$, then we have
\begin{align*}
    \eta |x_{i, t}^{\top}U_{i, 2,t}^{\top}U_{i, 2, t} v_2| \pm O(\eta^2 |x_{i, t}^{\top}v_{i, 2,t} - 1| D^3) \leq &~ \eta \cdot \tilde{O}( Dd\sigma) \cdot (1- \frac{\eta}{4D})^{t-T_1} + O(\eta^2D^3)\cdot (1 - \frac{\eta}{4D})^{t-T_1} \\
    \leq &~ \frac{1}{2}(1 - \frac{\eta}{4D})^{t+1 - T_1}\cdot \frac{\eta}{4D},
\end{align*}
where the first step holds due to induction hypothesis and
\[
|x_{i, t}^{\top}U_{i, 2,t}^{\top}U_{i, 2, t} v_{i, 2}|  \leq \|x_{i, t}^{\top}U_{i, 2,t}^{\top}\|_2 \|U_{i, 2, t} v_{i, 2,t} \|_2  \leq \tilde{O}(Dd\sigma) \cdot(1- \frac{\eta}{4D})^{t-T_1}.
\]
Therefore
\begin{align*}
   |1 - x_{i, t+1}^{\top}v_{i, 2,t+1}| \leq \frac{1}{2}(1- \frac{\eta}{4D})^{t+2 - T_1} + \frac{1}{2}(1 - \frac{\eta}{4D})^{t+1 - T_1}\cdot \frac{\eta}{4D} = \frac{1}{2}(1 - \frac{\eta}{4D})^{t+1 - T_1}.
\end{align*}

Next, we prove the second claim. We have
\begin{align}
    &~ \|U_{i, 2,t+1}v_{i, 2,t+1}\|_2 \notag \\
    = &~ \|(U_{i, 2,t} - \eta U_{i, 2,t} v_{i, 2,t}v_{i, 2,t}^{\top})(v_{i, 2,t} - \eta x_{i, t}\|w_{i, B}\|_2^2(x_{i, t}^{\top}v_{i, 2,t} - 1) - \eta U_{i, 2,t}^{\top}U_{i, 2,t}v_{i, 2,t}) \|_2 \notag \\
    \leq &~ \|U_{i, 2,t}v_{i, 2,t}(1 - \eta v_{i, 2,t}^{\top}v_{i, 2,t}) - \eta U_{i, 2,t}U_{i, 2,t}^{\top}U_{i, 2,t}v_{i, 2,t}\|_2 + \eta\|w_{i, B}\|_2^2 |x_{i, t}^{\top}v_{i, 2,t} - 1| \| U_{i, 2,t}x_{i, t} \|_2 \notag \\
    &~ \pm O(\eta^2 D^3) \cdot \|U_{i, 2,t}v_{i, 2,t}\|_2\notag \\
    \leq &~ (1 - 5\eta \|v_{i, 2,t}\|_2^2/6)\|U_{i, 2,t}v_{i, 2,t}\|_2 + \eta\|w_{i, B}\|_2^2|x_{i, t}^{\top}v_{i, 2,t} -  1|\| U_{i, 2,t}x_{i, t} \|_2 \notag \\
    \leq &~ (1 - \frac{\eta}{3D})\|U_{i, 2,t}v_{i, 2,t}\|_2 + \eta\|w_{i, B}\|_2^2|x_{i, t}^{\top}v_{i, 2,t} -  1|\| U_{i, 2,t}x_{i, t} \|_2 , \label{eq:indc_5}
\end{align}
where the first step follows from the gradient update rule (Lemma \ref{lem:grad}), the second step holds due to triangle inequality and 
\[
\|v_{i, 2,t}\|_2 \leq O(D), \quad |x_{i, t}^{\top}v_{i, 2,t} - 1|\|w_{i, B}\|_2^2\|x_{i, t}\|_2 \leq O(D^2) \quad \text{and} \quad \|U_{i, 2,t}^{\top}U_{i, 2,t}v_{i, 2,t}\| \ll 1,
\]
the third step holds due to $\eta D^3 \leq \|v_{i, 2,t}\|_2^2/6$ and the last step holds since 
\begin{align*}
\|v_{i, 2,t}\|_2^2 = &~ v_{i, 2,t}^{\top} (\|w_{i, B}\|_2x_{i, t} + v_{i, 2,t} - \|w_{i, B}\|_2 x_{i, t} )\\
\geq &~ \|w_{i, B}\|_2 x_{i, 2,t}^{\top}v_{i, 2,t} - \|v_{i, 2,t}\|_2 \|\|w_{i, B}\|_2 x_{i, t} - v_{i, 2,t}\|_2 \\
\geq &~ \frac{1}{2D} - \tilde{O}(Dr\sigma) \geq \frac{2}{5D}.
\end{align*}
{\noindent \bf Case 1.} Suppose $(1 - \frac{\eta}{4D})^{t+2 - T_1} \leq \|U_{i, 2,t}v_{2,t}\|_2 \leq (1 - \frac{\eta}{4D})^{t - T_1}$, then
\begin{align*}
    \|U_{i, 2,t+1}v_{i, 2,t+1}\|_2 \leq &~ (1 - \frac{\eta}{3D})\|U_{i, 2,t}v_{i, 2,t}\|_2 + \eta\|w_{i, B}\|_2^2|x_{i, t}^{\top}v_{i, 2,t} -  1|\| U_{i, 2,t}x_{i, t} \|_2 \\
    \leq &~ (1- \frac{\eta}{4D})\|U_{i, 2,t}v_{i, 2,t}\|_2 \leq (1 - \frac{\eta}{4D})^{t+1- T_1},
\end{align*}
where the first step comes from Eq.~\eqref{eq:indc_5}, the second step comes from 
\begin{align*}
\eta\|w_{i, B}\|_2^2|x_{i, t}^{\top}v_{i, 2,t} -  1|\| U_{i, 2,t}x_{i, t} \|_2 \leq &~ \eta D^2 \cdot \frac{1}{2}(1- \frac{\eta}{4D})^{t - T_1} \cdot \tilde{O}(Dd\sigma)\\
\leq &~ \frac{\eta}{12D}(1 - \frac{\eta}{4D})^{t+2- T_1} \leq \frac{\eta}{12D}\|U_{i, 2,t}v_{i, 2,t}\|_2.
\end{align*}

{\noindent \bf Case 2.} Suppose $\|U_tv_{2,t}\|_2 \leq (1 - \frac{\eta}{4D})^{t+2 - T_1}$, then
\begin{align*}
    \|U_{i, 2,t+1}v_{i, 2,t+1}\|_2 \leq &~ \|U_{i, 2,t}v_{i, 2,t}\|_2 + \eta\|w_{i, B}\|_2^2\cdot|x_{i, t}^{\top}v_{i, 2,t} -  1|\cdot\|U_{i, 2,t}x_{i, t} \|_2 \\
    \leq &~ (1 - \frac{\eta}{4D})^{t+2 - T_1} + \frac{1}{2}\eta D^2 (1 - \frac{\eta}{4D})^{t - T_1}\cdot \tilde{O}(Dd\sigma)\\
    \leq &~ (1 - \frac{\eta}{4D})^{t+2 - T_1} + (1 - \frac{\eta}{4D})^{t+1- T_1} \cdot \frac{\eta}{4D}\\
    = &~ (1 - \frac{\eta}{4D})^{t+1- T_1},
\end{align*}
where the first step comes from Eq.~\eqref{eq:indc_5}, the second step follows from the induction hypothesis and $\|U_{i, 2,t}x_{i, t}\|_2^2 \leq \tilde{O}(Dd\sigma)$. We have proved the second claim.

Now we move to the third claim. One has
\begin{align*}
    &~ \|\|w_{i, B}\|_2x_{i, t + 1} - v_{i, 2, t + 1}\|_2^2 - \|\|w_{i, B}\|_2 x_{i, t} - v_{i, 2, t}\|_2^2\\
    = &~ 2\eta(x_{i, t}^{\top}v_{i, 2,t} - 1)  \|w_{i, B}\|_2\|\|w_{i, B}\|_2x_{i, t} - v_{i, 2,t}\|_2^2 +\eta\langle \|w_{i, B}\|_2 x_{i, t} - v_{i, 2,t}, U_{i, 2,t}^{\top}U_{i, 2,t}v_{i, 2,t}\rangle \pm O(\eta^2 D^4)\\
    \lesssim &~ 2\eta\cdot  d^2D^3\sigma^2 \cdot D \cdot (r\sigma)^2 + \eta \cdot r\sigma \cdot d^2 \sigma^2D + \eta^2 D^4\\
    \lesssim &~ \eta D d^2r\sigma^3.
\end{align*}
The first step comes from Eq.~\eqref{eq:indc_mid}, the second step follows from 
\begin{align*}
     \|w_{i, B}\|_2\leq D, \quad \|\|w_{i, B}\|_2x_{i, t} - v_{i, 2,t}\|_2\leq \tilde{O}(r\sigma), \quad \|U_{i, 2,t}^{\top}U_{i, 2,t}v_{i, 2,t}\|_2 \leq \tilde{O}(d^2\sigma^2 D)
\end{align*}
and 
\begin{align*}
      x_{i, t}^{\top}v_{i, 2,t} - 1 \leq  \tilde{O}(d^2D^3\sigma^2).
\end{align*}
Here the last term holds since (i) $|x_{i, \tau+1}^{\top}v_{i, 2,\tau+1} - x_{i, \tau}^{\top}v_{i, 2,\tau}| \leq O(\eta D)$, i.e., the step size is at most $\eta D$ (see Eq. \eqref{eq:grad_xv} \eqref{eq:smooth}); (ii) $x_{i, T_1}^{\top}v_{i, 2, T_1} < 1$ and (iii) $|x_{i, \tau+1}^{\top}v_{i, 2,\tau+1} - 1| < |x_{i, \tau}^{\top}v_{i, 2,\tau} - 1|$ whenever
\[
\eta \|x_{i, t}^{\top}U_{i, 2,t}^{\top} U_{i, 2,t}v_2\|_2 \leq \eta \cdot \tilde{O}(d^2D^2\sigma^2) \lesssim \frac{\eta}{2D} |x_{i, \tau}^{\top}v_{i, 2,\tau} - 1|  \quad \Rightarrow \quad |x_{i, \tau}^{\top}v_{i, 2,\tau} - 1| \gtrsim d^2D^3\sigma^2.
\]
That is, combining (i) (ii), we know that the first time $x_{i, \tau}^{\top}v_{i, \tau}$ being greater $1$ must obey $x_{i, \tau}^{\top}v_{i, \tau} < 1 + O(\eta D)$, (iii) implies that whenever $x_{i, \tau+1}^{\top}v_{i, 2,\tau+1} - 1 \gtrsim d^2D^3\sigma^2$, it value should decrease, hence we conclude 
\[
x_{i, T_1}^{\top}v_{i, 2, T_1} - 1 \lesssim \eta D + d^2 D^3 \sigma^2 \lesssim d^2 D^3 \sigma^2.
\]

Taking a telescopic summation, one has
\begin{align*}
\|\|w_{i, B}\|_2x_{i, t} - v_{i, 2, t}\|_2^2 - \|\|w_{i, B}\|_2 x_{i, T_1} - v_{i, 2, T_1}\|_2^2 \leq &~  (t - T_1)\cdot O(\eta D d^2r^2\sigma^3)\\
\leq &~  \tilde{O}(D^2d^2r \sigma^3) \leq r^2\sigma^2.
\end{align*}
This concludes the third claim. We conclude the proof here.
\end{proof}

\subsection{Missing proof from Section \ref{sec:induction}}

\begin{proof}[Proof of Theorem \ref{thm:linear}]
Due to the reduction established in Section \ref{sec:reduction}, it suffices to prove Eq.~\eqref{eq:prob-online} and Eq.~\eqref{eq:prob-online1}.
For each environment $i$ ($i \in [k]$), we inductively prove
\begin{enumerate}
    \item $\DPGD$ achieves good accuracy on the current environment, i.e., $\|U_{i, T}v_i - w_i\|_2 \leq \eps \nu$;
    \item The feature matrix $U_{i}$ remains well conditioned, i.e. $\frac{1}{2\sqrt{D}} \leq \sigma_{\min}(U_{i, \mathsf{end}}) \leq\sigma_{\max}(U_{i, \mathsf{end}}) \leq 2\sqrt{D}$.
    \item The algorithm does not suffer from catastrophic forgetting, i.e., $\|U_{i, t}v_j - w_i\|_2 \leq \eps$ for any $j < i$ and $t\in [T]$;
\end{enumerate}

The base case ($i=0$) holds trivially as at the beginning of CL, we have $\W, \V = \emptyset$ and $U = 0$. 
Suppose the induction holds up to the $(i - 1)$-th environment, we focus on the second and last claim, as the first claim holds directly due to Lemma \ref{lem:convergence}. 

For the second claim, we have already proved $\|U_{i, T}v_i - w_i\|_2 \leq \eps\nu$, this indicates that each coordinate of $U_{i, T}v_i - w_i$ is less than $\nu/2$. Since we assume each coordinate of $w_i$ is a multiple of $\nu$, therefore, we have $\hat{w}_i = \mathsf{Round}_{\nu}(U_{i, T} v_i) = w_i$. That is, we exact recover $w_i$. We divide into two cases. 

{\noindent \bf Case 1. } If $\|w_{i, B}\|_2 = 0$, i.e., $w_{i} \in \W$, then $\|P_{\W_{\perp}}\hat{w}_i\|_2 = \|P_{\W_{\perp}}w_i\|_2 = 0$, Therefore, we do not update $\W$ and $\V$, and 
\[
U_{i, \mathsf{end}} = P_{\W}U_{i, T}P_{\V} = P_{\W}(U_{i, A, 0} + U_{i, B, T})P_{\V} = P_{\W}U_{i, A, 0}P_{\V} = U_{i-1, \mathsf{end}},
\]
where the second and the third step holds to Lemma \ref{lem:decomp} and the last step just holds due to definition. Hence $U_{i}$ continues to be well-conditioned (since it does not change).

{\noindent \bf Case 2. } If $\|w_{i, B}\|_2 \in [1/D, D]$, then $\|P_{\W_{\perp}}\hat{w}_i\|_2 = \|P_{\W_{\perp}}w_i\|_2 = \|w_{i,B}\| \geq 1/D$. Hence, we augment $\W_{i} = \W_{i-1} \cup \{w_i\}$ and $\V_{i} = \V_{i-1} \cup \{v_i\}$ and have
\begin{align}
U_{i, \mathsf{end}} = &~ P_{\W}U_{i, T}P_{\V} = P_{\W}(U_{i, A, 0} + U_{i, B, T})P_{\V}\notag\\
=&~ U_{i, A, 0} + (\frac{1}{\|w_{i, B}\|_2^2}w_{i, B} w_{i, B}^\top) U_{i, B, T} (\frac{1}{\|v_{i,2, T}\|_2^2}v_{i,2,T} v_{i,2, T}^\top)\notag\\
= &~ U_{i, A, 0} + (\frac{1}{\|w_B\|_2^2}w_{B} w_{B}^\top) (w_B x_{i, T}^{\top} + U_{i,2,T})  (\frac{1}{\|v_{i,2,T}\|_2^2}v_{i,2,T} v_{i,2,T}^\top)\notag \\
= &~ U_{i, A, 0} + w_Bv_{i,2,T}^{\top} \frac{ x_{i, T}^{\top} v_{i,2,T}}{\|v_{i,2,T}\|_2^2}\notag \\
= &~ U_{i, A, 0} + (1 \pm o(\eps/D))\frac{1}{\|v_{i,2,T}\|_2^2} w_Bv_{i,2,T}^{\top} \label{eq:condition1} .
\end{align}
The third step holds since $\row(U_{i, A, 0}) \in \V$, $\column(U_{i, A, 0}) \in \W$, $\column(U_{i, B, T}) \cap \W = w_{i,B}$, $\row(U_{i, B, T}) \cap \V = v_{i, 2, T}$ (see Lemma \ref{lem:decomp}), the later two imply the projection operation essentially boils to projection on $w_{i, B}$ and $v_{i, 2, T}$.
The fifth step follows from $\column(U_{i,2,T}) \perp w_B$ (see Lemma \ref{lem:decomp}), the sixth step follows from $x_{i, T}^{\top}v_{i, T} = 1 \pm o(\eps/D)$ (see Lemma \ref{lem:stage_two}). To bound the condition number, it suffices to note that $w_{i, B} \perp \W_{i-1}$, $v_{i, 2, T} \perp \V_{i-1}$ (see Lemma \ref{lem:decomp}), and therefore, $w_{B}\perp \column(U_{i, A, 0})$, $v_{i,2,T}\perp \row(U_{i, A, 0})$ (i.e., we add an orthogonal basis) and 
\[
(1 \pm o(\eps/D))\frac{1}{\|v_{i,2,T}\|_2^2} \|w_B\|_2 \|v_{i,2,T}^{\top}\|_2 = (1 \pm o(\eps/D))\frac{\|w_B\|_2}{\|v_{i,2,T}\|_2}  = (1+o(1))\sqrt{\|w_B\|_2} \in \left[\frac{1}{2\sqrt{D}}, \frac{\sqrt{D}}{2}\right]
\]
where the last step is derived from $x_{i,t}^{\top}v_{i,2,T} \approx 1 + o(1)$ and $\|\|w_B\|_2x_{i, t}- v_{i, 2,t}\|_2 \approx 1 \pm o(1/D^3)$.
We have proved the second claim.

For the last claim, fix an index $j < i$, we prove the accuracy of $j$-th environment would not drop significantly and remain good. Note by inductive hypothesis, we already have $\|U_{j, T}v_j - w_j\|_2 \leq \eps\nu/kd$ before the final projection step of $j$-th environment. After the projection step, one has
\begin{align*}
    \|U_{j, \mathsf{end}}v_{j} - w_j\|_2 = \|P_{\W}U_{j, T}P_{\V}v_j - w_j\|_2 = \|P_{\W}(U_{j, A, T} + w_{j, B}x_{j, T}^{\top} + U_{j, 2,T})P_{\V}v_j - w_j\|_2
\end{align*}
We divide into two cases.

{\noindent \bf Case 1.} Suppose $\|w_{j, B}\|_2 = 0$. We have $\W_j = \W_{j-1}, \V_{j} = \V_{j-1}$ and
\begin{align*}
    \|U_{j, \mathsf{end}}v_j - w_j\|_2 = &~ \|P_{\W_j}(U_{j, A, T} + U_{j, 2,T})P_{\V_j}v_j - w_j\|_2 = \|U_{j, A, T} v_j - w_j\|_2 \\
    \leq &~ \|U_{j, A, T} v_{j, 1, T} - w_j\|_2 \leq \eps\nu.
\end{align*}
The second step follows from $\column(U_{j, A, T}) \in \W_{j}$, $\row(U_{j, A, T}) \in \V_{k}$ and $\row(U_{j, 2, T}) \in \V_{j,\perp}$ (see Lemma \ref{lem:decomp}), the third step follows from $\row(U_{j, A, T}) \in \V_{i-1}$. Hence, we have that the error remains small after the projection.

During the $i$-th environment, for any $t \in [T]$, we decompose $U_{i, t} = U_{j, \mathsf{end}} + \hat{U}_{i, t}$.
We have
\begin{align*}
    \|U_{i, t}v_j - w_j\|_2 = &~ \|(U_{j, \mathsf{end}} + \hat{U}_{i,t})v_j - w_j\|_2 \\
    \leq &~ \|U_{j, \mathsf{end}}v_j - w_j\|_2 + \|\hat{U}_{i,t}v_{j}\|_2\\
    = &~ \|U_{j, \mathsf{end}}v_j - w_j\|_2 + \|\hat{U}_{i,t}v_{j, 2, T}\|_2\\
    \leq &~ \eps\nu + \tilde{O}(\sqrt{D}\cdot r\sigma)\\
    \leq &~ \eps.
\end{align*}

 The third step holds due to the fact that $\mathsf{row}(\hat{U}_{i, t}) \in V_{j, \perp}$, the fourth step holds due to (1)
 $\|v_{j, 2,t}\|_2$ is non-decreasing during the $j$-th environment (see the gradient update formula in Lemma \ref{lem:grad}) and therefore $\|v_{j, 2,T}\|_2 \leq \|v_{j, 2, 0}\|_{2} \leq \tilde{O}(r\sigma)$ w.h.p.; (2) the spectral norm $\|\hat{U}_{i, t}\| \leq O(\sqrt{D})$, since
 \begin{align*}
     \|\hat{U}_{i, t}\| \leq &~ \|U_{i,t}\| + \|U_{j, \mathsf{end}}\|_2
     \leq \|U_{i, A}\| + \|w_{i, B}x_{i, t}^{\top} + U_{i, 2, T}\| + \|U_{j, \mathsf{end}}\|_2\\
     \leq &~ 2\sqrt{D} + 2\sqrt{D} + 2\sqrt{D} = O(\sqrt{D}).
 \end{align*}
Here the first step and the second step hold due to triangle inequality, the second step holds due to the inductive hypothesis and $\|w_{i, B}x_{i, t}^{\top} + U_{i, 2, T}\|  \leq 2\sqrt{D}$. We finished the proof of the first case.

{\noindent\bf Case 2. } Suppose $\|w_{j,B}\|_2 \in [1/D, D]$. Then we augment $\W_{j} = \W_{j-1} \cup \{w_j\}$ and $\V_{j} = \V_{j-1} \cup \{v_j\}$. We first prove the loss remains small after the final projection step of $j$-th environment. In particular, we have
\begin{align*}
\|U_{j, \mathsf{end}}v_{j} - w_j\|_2 =&~ \|\big(U_{j, A, 0} + (1\pm o(\eps/D))\frac{1}{\|v_{j, 2, T}\|_2^2} w_{B} v_{j, 2, T}^{\top}\big) (v_{j, 1, T} + v_{j,2, T}) - w_{j, A} - w_{j, B} \|_2\\
= &~\|(U_{j, A, 0}v_{j, 1, T} - w_{j, A}) +  (1\pm o(\eps/D))\frac{1}{\|v_{j, 2, T}\|_2^2} w_{j, B} v_{j, 2, T}^{\top}v_{j, 2, T}  - w_{j, B} \|_2\\
\leq &~ \|(U_{j, A, 0}v_{j, 1, T} - w_{j, A})\|_2 + o(\eps/D)\|w_{j, B}\|_2\\
\leq &~ \eps\nu + o(\eps)
\leq  \eps.
\end{align*}
The first step holds due to Eq. \eqref{eq:condition1}, the third step holds due to triangle inequality, the fourth step holds due to the inductive hypothesis and $\|w_{j, B}\|_2\leq D$.

During the $i$-th environment, since the update is performed in the orthogonal space, we expect $Uv_{j}$ does not change. Formally, let $U_{i, t} = U_{j, \mathsf{end}} + \hat{U}_{i,t}$, where $\column(\hat{U}_{i, t}) \perp \W_j$ and $\row(\hat{U}_{i,t}) \perp \V_j$, then 
\[
U_{i, t}v_{j} = (U_{j, \mathsf{end}} + \hat{U}_{i,t})v_j = U_{j, \mathsf{end}}v_j,
\]
Hence $\|U_{i, t}v_{j} - w_j\|\leq \eps$ continues to hold. We conclude the proof here.

\end{proof}

%% file: appendix-lower.tex
\section{Missing proof from Section \ref{sec:lower}}
\label{sec:lower-app}

We provide the proof of a technical Lemma used in proving Theorem \ref{thm:lower}
\begin{lemma}[Technical tool]
\label{lem:tech}
Let $\Pi_{n}^{d}$ be the space of all polynomial of degree at most $d$ in $n$ variables. For any two polynomials $p_1(x), p_2(x) \in \Pi_{3}^{2}$, if
\[
\E_{x\sim \mathcal{B}_3(0, 1)}[(p_1(x) - p_2(x))^2] \leq \frac{1}{1000},
\] 
then the absolute deviation of each coefficient is at most $1/4$.
\end{lemma}
\begin{proof}
Let $p(x) = (p_1(x) - p_2(x))^2$, taking an integral over $B_{3}(0,1)$, we can only need to consider all quadratic terms, since all odd terms would be canceled due to symmetry. We divide into cases. (1) The coefficient of the constant term is greater than $1/4$, then $p(x) \geq 1/16$. (2) The coefficient of $x_1$ is greater than $1/4$, then $p(x) \geq \E_{x\sim \mathcal{B}_3(0, 1)} \tfrac{1}{16} x_1^2 = \frac{1}{16} \cdot \frac{1}{5} = \frac{1}{80}$. (3) The coefficient of $x_1x_2$ is greater than $1/4$, then then $p(x) \geq \E_{x\sim \mathcal{B}_3(0, 1)} \tfrac{1}{16} x_1^2x_2^2 = \frac{1}{16} \cdot \frac{1}{35} = \frac{1}{560}$. (4) The coefficient of $x_1^2$ is greater than $1/4$, then then $p(x) \geq \E_{x\sim \mathcal{B}_3(0, 1)} \tfrac{1}{16} x_1^4 = \frac{1}{16}\cdot \frac{3}{35} = \frac{3}{560}$. Hence we conclude no coefficient has difference greater than $1/4$.
\end{proof}